\documentclass{article} % For LaTeX2e
\usepackage{iclr2026_conference,times}

\usepackage{amsmath,amsfonts,bm}

\def\eqref#1{equation~\ref{#1}}
\def\1{\bm{1}}

\DeclareMathAlphabet{\mathsfit}{\encodingdefault}{\sfdefault}{m}{sl}
\SetMathAlphabet{\mathsfit}{bold}{\encodingdefault}{\sfdefault}{bx}{n}

\usepackage{hyperref}
\usepackage{url}

\usepackage[textsize=tiny]{todonotes}
\usepackage{adjustbox}
\usepackage{bm}
\usepackage{graphicx}
\usepackage{caption}
\usepackage[labelformat=simple]{subcaption}

\usepackage{wrapfig}
\usepackage{diagbox}
\usepackage{pifont}
\usepackage{amsmath}
\usepackage{amssymb}
\usepackage{enumitem}
\usepackage{mathtools}
\usepackage{amsthm}
\usepackage{makecell}
\usepackage[ruled,vlined]{algorithm2e}
\usepackage[utf8]{inputenc} % allow utf-8 input
\usepackage[T1]{fontenc}    % use 8-bit T1 fonts
\usepackage{hyperref}       % hyperlinks
\usepackage[capitalize,noabbrev]{cleveref}
\usepackage{url}            % simple URL typesetting
\usepackage{booktabs}       % professional-quality tables
\usepackage{amsfonts}       % blackboard math symbols
\usepackage{nicefrac}       % compact symbols for 1/2, etc.
\usepackage{microtype}      % microtypography
\usepackage{xcolor}         % colors

\theoremstyle{plain}

\newtheorem{proposition}{Proposition}

\theoremstyle{definition}

\theoremstyle{remark}

\title{Guided Policy Optimization under Partial Observability}
\author{Yueheng Li$^{1}$, Guangming Xie$^{1*}$, Zongqing Lu$^{2}$\thanks{Corresponding authors.} \\
$^{1}$School of Advanced Manufacturing and Robotics, Peking University \\
$^{2}$School of Computer Science, Peking University \\
\texttt{\{liyueheng,xiegming,zongqing.lu\}@pku.edu.cn} \\
}

\iclrfinalcopy % Uncomment for camera-ready version, but NOT for submission.
\begin{document}

\maketitle

\begin{abstract}
Reinforcement Learning (RL) in partially observable environments poses significant challenges due to the complexity of learning under uncertainty. 
While additional information, such as that available in simulations, can enhance training, effectively leveraging it remains an open problem. 
To address this, we introduce Guided Policy Optimization (GPO), a framework that co-trains a guider and a learner. 
The guider takes advantage of privileged information while ensuring alignment with the learner's policy that is primarily trained via imitation learning. 
We theoretically demonstrate that this learning scheme achieves optimality comparable to direct RL, thereby overcoming key limitations inherent in existing approaches. 
Empirical evaluations show strong performance of GPO across various tasks, including continuous control with partial observability and noise, and memory-based challenges, significantly outperforming existing methods.
\end{abstract}

\section{Introduction}

Many real-world tasks can be formulated as sequential decision-making problems where agents take actions in an environment to achieve specific goals over time \citep{puterman2014markov}. 
Reinforcement Learning (RL) has emerged as a powerful tool for solving such tasks, leveraging trial-and-error learning to optimize long-term rewards \citep{Sutton1998}. 
Despite its success, RL encounters significant hurdles in complex and partially observable environments, where agents often operate with limited or noisy information \citep{madani1999undecidability}. 
However, during training, we often have access to extra information that could significantly enhance learning efficiency and performance
\citep{doi:10.1126/scirobotics.abc5986,pmlr-v164-chen22a}.
For instance, in robotics, while real-world sensor data may be noisy or incomplete, simulation environments typically provide full state observability.

Despite the potential of such privileged information, effectively leveraging it in practice remains a major challenge. 
One popular strategy to utilize this information is through methods like Imitation Learning (IL) \citep{hussein2017imitation}, Teacher-Student Learning (TSL), or policy distillation \citep{pmlr-v89-czarnecki19a}. 
In these approaches, a teacher, equipped with privileged information, provides supervision to guide the student’s learning process.
However, this strategy introduces its own set of challenges: a teacher with privileged information may impose an unrealistically high-performance standard, making it difficult for the student to effectively imitate.
This issue, known as the “impossibly good” teacher \citep{Walsman2023ImpossiblyGE} or imitation gap \citep{10.5555/3540261.3541724}, can hinder learning and degrade performance. 
To address this, previous work has sought to integrate environmental rewards into the learning process of the student.  
One approach is to combine RL with IL \citep{10.5555/3540261.3541724, pmlr-v202-shenfeld23a, pmlr-v205-nguyen23a}, switching to RL-based training when the teacher becomes inimitable. 
Another approach modifies environmental rewards based on the teacher through policy distillation \citep{pmlr-v89-czarnecki19a, Walsman2023ImpossiblyGE}. 
However, such methods diminish the utility of privileged information, 
%and the pre-trained teacher can be costly, 
often resulting in inefficient use of the teacher's knowledge.

To better exploit available information, we propose training a ``possibly good" teacher, i.e., a teacher whose policy remains within the learner’s imitable region.
Inspired by Guided Policy Search (GPS) \citep{levine2013guided, NIPS2016_a00e5eb0}, we introduce Guided Policy Optimization (GPO), a novel framework that trains both the teacher and the student simultaneously while ensuring that the teacher’s policy remains aligned with that of the student. The key insight behind GPO is that by leveraging privileged information during training, the teacher can be trained more effectively while ensuring that its performance is ``possibly good," thus facilitating easier imitation by the student.
Theoretically, we show that the student can achieve optimality similar to direct RL training, mitigating the suboptimality and imitation gaps that often arise from purely teacher-based supervision.
We empirically validate our algorithm across various tasks, including didactic examples, challenging continuous control tasks in partially observable, noisy environments within the Brax \citep{brax2021github} domain, and in memory-based tasks from the POPGym \citep{morad2023popgym} benchmark.
GPO shows consistent and significant improvements, underscoring its ability to exploit extra information and deliver robust performance across diverse domains.

\section{Background}
We consider Partially Observable Markov Decision Process (POMDP) \citep{KAELBLING199899}, which is characterized by the tuple $\langle \mathcal{S},\mathcal{A},r,\mathcal{P},\mathcal{O},\gamma\rangle$. $\mathcal{S}$ represents the set of states, $\mathcal{A}$ the set of actions, $r$ the reward function, $\mathcal{P}$ the transition probability function, $\mathcal{O}$ the partial observation function and $\gamma$ the discount factor.
At each time step $t$, the agent receives a partial observation $o_t\sim\mathcal{O}(\cdot|s_t)$ for current state $s_t\in\mathcal{S}$. 
The agent then selects an action $a_t\in\mathcal{A}$ according to $o_t$ or its action-observation history $\tau_t:\{o_0,a_0,o_1,a_1...,o_t\}$. The state transitions to the next state $s_{t+1}$ according to $\mathcal{P}(s_{t+1}|s_t,a_t)$, and the agent receives a reward $r_t$. 
The goal for the agent is to find the optimal policy $\pi^*:\tau\to\Delta(\mathcal{A})$ that maximizes the return, expressed as $\pi^*=\text{arg}\max_\pi V_\pi$, where $V_\pi=\mathbb{E}[\sum_{t=0}^\infty\gamma^tr_t|\pi]$ represents cumulative rewards.
When full state information $s$ is available during training, we may also define a policy $\mu: s \to \Delta(\mathcal{A})$ based on privileged information. 
For clarity, throughout this paper we refer to such privileged training inputs simply as the state $s$, though in practice they could take other forms.
Likewise, we refer to partial observations simply as $o$, though in practice they may include histories or other derived features.

Finally, we emphasize that in the remainder of this paper, the term “optimal” refers to the student’s optimal policy under partial observability—not the teacher’s optimal policy under privileged information, which is generally unattainable for the student.

\subsection{Teacher-Student Learning}
Since we consider both training the teacher and student, 
in this paper, we use the term Teacher-Student Learning (TSL) to broadly refer to Imitation Learning (IL) \citep{hussein2017imitation}, policy distillation \citep{pmlr-v89-czarnecki19a}, and related approaches, as there is no fundamental distinction between them. 
In TSL, the teacher policy is typically pre-trained using RL or derived from other methods such as a classical controller, which is assumed to effectively accomplish the desired task. 
The goal is for the teacher to somehow provide supervision to the student in learning to solve the same task.

A straightforward approach to training the agent is to directly supervise the student’s policy $\pi$ using the teacher's policy $\mu$, similar to Behavioral Cloning (BC) \citep{pomerleau1991efficient,torabi2018behavioral}:
\begin{equation}
    \min_\pi \mathbb{E}_{s\sim d_\mu}[\text{D}_{\text{KL}}(\mu(\cdot|s),\pi(\cdot|s))],
\end{equation}
where $d_\mu$ is the distribution of states under the teacher’s policy, and $\text{D}_{\text{KL}}$ is the Kullback-Leibler (KL) divergence. This objective encourages the student’s policy to mimic the teacher’s policy for the observed states.
However, when the teacher’s policy is based on privileged information, the student can only learn the statistical average of the teacher’s actions \citep{Warrington2020RobustAL,10.5555/3540261.3541724},
and be strictly suboptimal \citep{NEURIPS2024_74d188c5}.
In this paper, we refer to such a teacher as \textit{inimitable}, and we highlight this limitation through two illustrative examples in the next subsection.

\subsection{Didactic Examples}\label{sec:dida}

\begin{table}[h]
    \centering
    \begin{minipage}{0.45\textwidth}
        \centering
        \begin{tabular}{|c|c|c|c|}
\hline
\diagbox{state}{action} & $a_L$ & $a_R$ & $a_l$ \\
\hline
$s_L$   & 1   & 0  & -0.1\\
\hline
$s_R$   & 0   & 1  & -0.1\\
\hline
\end{tabular}
        \caption{TigerDoor problem}\label{tab:tiger}
    \end{minipage}
    \hfill
    \begin{minipage}{0.45\textwidth}
        \centering
\begin{tabular}{|c|c|c|}
\hline
\diagbox{state}{action} & $a_L$ & $a_R$  \\
\hline
$s_L$   & 2   & 0  \\
\hline
$s_R$   & 0   & 1  \\
\hline
\end{tabular}
\caption{TigerDoor-alt problem}\label{tab:tiger'}
    \end{minipage}
\end{table}

\textbf{TigerDoor}.
In the classic TigerDoor problem \citep{10.5555/3091622.3091666}, there are two doors with a tiger hidden behind one of them. 
The possible state $s_L$ (tiger behind the left door) and $s_R$ (tiger behind the right door), with equal probabilities for each, form $\mathcal{S} = \{s_L, s_R\}$. 
The action set is $\mathcal{A} = \{a_L, a_R, a_l\}$, where $a_L$ and $a_R$ denote opening the left and right doors, respectively, and $a_l$ denotes listening to determine the tiger's location. 
The teacher knows the tiger's location whereas the student can only ascertain it after choosing $a_l$.
The payoff matrix is shown in Table \ref{tab:tiger}.
The optimal policy for the teacher is to always choose the correct door without listening, whereas the student's optimal strategy involves first listening to locate the tiger.
Consequently, the student cannot learn the optimal policy through supervision from the teacher, as the teacher never chooses $a_l$.
Under the teacher's supervision, the student will only learn to randomly select between $a_L$ and $a_R$, resulting in an expected reward of 0.5.
This scenario poses challenges for the supervised student, as the teacher fails to explore and gather essential information for the learner.

\textbf{TigerDoor-alt}.
We also introduce an alternative version of the problem, TigerDoor-alt (Table~\ref{tab:tiger'}), which further illustrates the imitation gap, even when no exploratory actions are required. 
In this scenario, the listening action $a_l$ is removed, and the reward for correctly selecting the left door is increased to 2.
Similarly, the teacher continues to select the correct door, while the student learns to randomly choose between the two doors, yielding an expected reward of 0.75. 
However, the optimal policy for the student is to always choose the left door, which provides an expected reward of 1.
This discrepancy arises from the loss of information when converting the reward-based objective into a policy-supervised objective.

The current solution to this issue is to incorporate rewards into the student's learning process. To effectively utilize teacher supervision, there are two kinds of approaches.
The first type dynamically adjusts the weight of the supervision loss between the teacher and pure RL training. 
This allows the algorithm to switch to pure RL when the teacher is deemed inimitable \citep{10.5555/3540261.3541724, pmlr-v202-shenfeld23a, shenfeld2023tgrl}. However, such approaches fail to fully utilize privileged information and may waste the valuable, often expensive, pre-trained teacher.
The second kind incorporates teacher supervision into the reward signal, for instance, by using reward shaping via the teacher's value function \citep{Walsman2023ImpossiblyGE}. However, this supervision is indirect and may require additional learning. 
Crucially, to the best of our knowledge, none of the existing methods provide theoretical guarantees that teacher supervision will actually be beneficial.

Another research direction attempts to reconstruct privileged information from partial observations. 
However, such methods require the MDP to be \emph{decodable}~\citep{efroni2022provable}, which is clearly infeasible in the TigerDoor setting. 
For a more detailed discussion of related work, see Appendix~\ref{app:rw}.

\section{Method}
We present our Guided Policy Optimization (GPO) framework, which co-trains two entities: the guider and the learner, which we use to differentiate from existing TSL methods.
GPO iteratively updates both policies to ensure alignment. 
We then explore both the theoretical properties and practical implementation of GPO, introducing two variants: GPO-penalty and GPO-clip.

\subsection{From GPS to GPO}\label{sec:3.1}

Unlike direct policy search methods, GPS \citep{levine2013guided, NIPS2016_a00e5eb0} does not optimize policy parameters directly. 
Instead, it introduces an intermediate agent (guider) and employs trajectory optimization to learn a time-varying linear-Gaussian policy, which is then used to train a neural network policy (learner) through supervised learning. 
Although GPS is a model-based method and is not directly applicable in our setting, its idea of introducing an intermediate agent to guide policy learning can be extended to an RL algorithm in the context of POMDPs.

Specifically, since the guider is only used during training, it can access any type of privileged information. 
The key requirement is to ensure that the guider is imitable by the learner, which motivates us to introduce the GPO framework, which operates through the following four steps:
\begin{itemize}[leftmargin=1em]
\item \textbf{Data Collection}: Collect trajectories by executing the guider's policy, denoted as $\mu^{(k)}$. 
\item \textbf{Guider Training}: Update the guider $\mu^{(k)}$ to $\hat{\mu}^{(k)}$ according to RL objective $V_{\mu^{(k)}}$. 
\item \textbf{Learner Training}: Update the learner to $\pi^{(k+1)}$ by minimizing the distance $D(\pi,\hat{\mu}^{(k)})$. 
\item \textbf{Guider Backtracking}: Set $\mu^{(k+1)}(\cdot|s) = \pi^{(k+1)}(\cdot|o)$ for all states $s$ before the next iteration.
\end{itemize}

\begin{figure*}
    \centering
\includegraphics[width=0.8\linewidth]{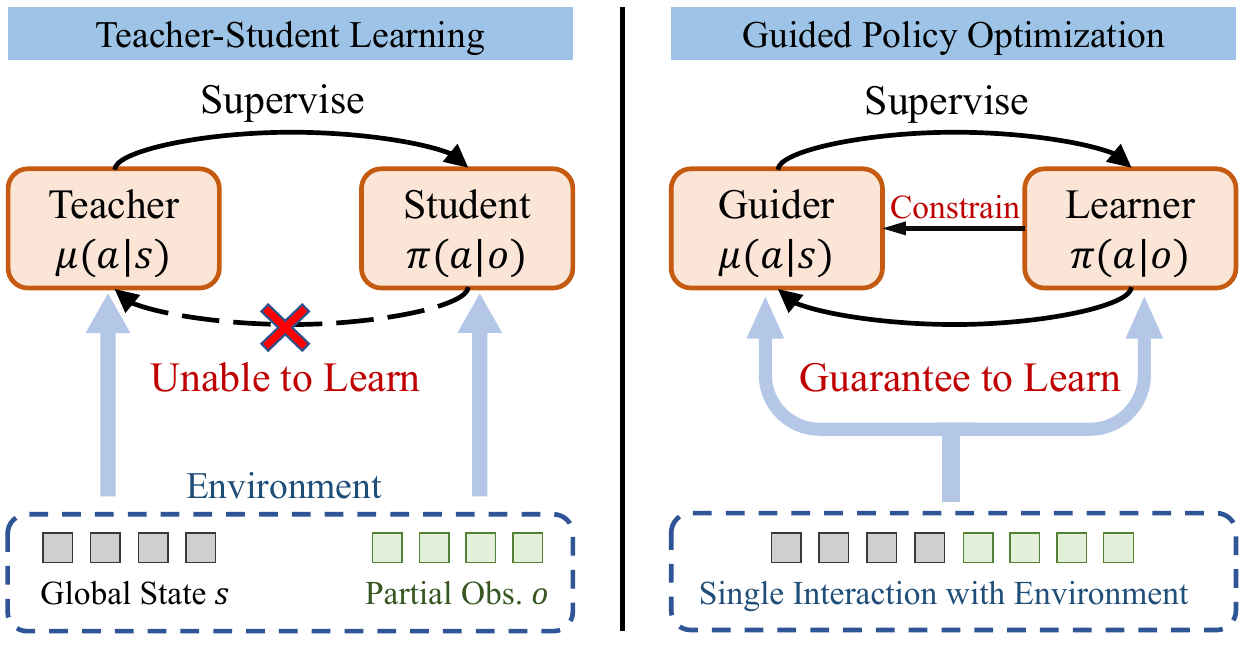}
\caption{The comparison between TSL and GPO.}
\label{fig:gpo}
\end{figure*}

In the learner training step, $D(\pi, \mu)$ can be any Bregman divergence. For this work, we use the KL divergence, weighted by the state distribution $d_\mu$.
GPO iterates these steps until convergence, applying standard RL to train the guider while the learner seeks to mimic the guider’s behavior. 
If the learner struggles due to discrepancies in observation spaces, the backtracking step adjusts the guider’s policy to mitigate the imitation gap.

The comparison between TSL and GPO is illustrated in Fig. \ref{fig:gpo}.
Several differences between the two frameworks exist.
First, the teacher in TSL is typically provided or trained independently from the student, while in GPO, the guider and learner are trained together. 
%This also means GPO does not require additional interaction with the environment.
Second, TSL typically allows the student to interact with the environment, whereas GPO only uses a guider, enabling more effective trajectory collection due to the behavioral policy being conditioned on privileged information.
Lastly, and most importantly, TSL does not use the student to constrain the teacher. 
This means that if the teacher is too advanced for the student, the student will struggle to learn from the teacher.

In contrast, GPO utilizes backtracking to guarantee the learner can effectively learn from the guider. This is demonstrated by the following proposition:
\begin{proposition}\label{prop}
If the guider's policy is updated using policy mirror descent in each GPO iteration:
\begin{equation*}
    \hat{\mu}=\text{arg}\min\{-\eta_k\langle\nabla V(\mu^{(k)}),\mu\rangle+\text{D}_{\mu^{(k)}}(\mu,\mu^{(k)})\},
\end{equation*}
where $\eta_k$ is the step size. Then the learner’s policy update follows a constrained policy mirror descent:
\begin{equation*}
    \pi^{(k+1)}=\underset{\pi\in\Pi}{\text{arg}\min}\{-\eta_k\langle\nabla V(\pi^{(k)}),\pi\rangle+\text{D}_{\pi^{(k)}}(\pi,\pi^{(k)})\}.
\end{equation*}
\end{proposition}
\begin{proof}
    See Appendix \ref{app:proof}.
\end{proof}
Here, we assume that the guider $\mu$ has access to an unlimited policy class, while the learner $\pi$ is constrained to a limited policy class $\Pi$ for simplicity. 
Policy mirror descent \citep{DBLP:journals/corr/abs-2005-09814, JMLR:v23:22-0056} is a general family of algorithms that encompasses a wide range of fundamental methods in RL, including trust-region algorithms like TRPO \citep{DBLP:journals/corr/SchulmanLMJA15} and PPO \citep{DBLP:journals/corr/SchulmanWDRK17}.
This proposition shows that, despite the learner not directly interacting with the environment, the GPO update for the learner can be viewed as a standard RL update. Specifically, if we use trust-region RL algorithms for the guider, the update for the learner’s policy inherits the key properties, such as policy improvement \citep{DBLP:journals/corr/SchulmanLMJA15}.
This suggests that GPO can effectively address challenges in TSL, such as dealing with a suboptimal teacher or the imitation gap, while still framing the learner’s policy as being supervised by the guider. 
In Appendix \ref{app:tiger}, we provide an intuitive example illustrating how GPO can achieve optimal in the TigerDoor-alt problem.

Given that GPO mirrors direct RL for the learner, one may ask: \textbf{What are GPO’s key advantages?}
The main benefit lies in leveraging additional information while simplifying learning.
Since policy gradients suffer from high variance especially under partial observability, GPO splits learning into two phases: the guider with privileged information handles complex RL gradients, while the partial observable learner is trained via an easier supervised learning, reducing variance and complexity.
For instance, to train robustness to noisy observations, GPO can train the guider on clean inputs and supervise the learner with noisy ones, resulting in a more stable and effective learning process.

%\subsection{Implementation of GPO}\label{sec:3.2}
\subsection{GPO-Penalty}\label{sec:penalty}
This section introduces a straightforward implementation of the GPO framework using KL-divergence as a penalty for the guider, which we refer to as GPO-penalty.
Specifically, in step 2 of GPO, we use PPO as the underlying trust-region algorithm. 
The corresponding loss for the guider's policy is as follows\footnote{We omit subscripts for expectations in the remainder of the paper, as all samples are drawn from the distribution induced by the behavioral policy $\beta = \mu_{\text{old}}$.}:
\begin{equation}\label{eq:PPO1}
    \mathcal{L}_1(\mu)=-\mathbb{E}\bigg[\min
    \bigg(\rho^\mu A^\beta(s,a),
    \rho_{clip}^\mu A^\beta(s,a)\bigg)\bigg],
\end{equation}
where $\rho^\mu=\mu(a|s)/\beta(a|s)$, $\rho_{clip}^\mu=clip(\rho^\mu,1-\epsilon,1+\epsilon)$ and $\beta$ denotes the behavioral policy. The advantage $A^\beta(s,a)$ is estimated using the Generalized Advantage Estimation (GAE) \citep{Schulman2015HighDimensionalCC} with the value function $V(s)$ trained via discounted reward-to-go.

In step 3, since finding the exact minimizer of the distance measure is computationally prohibitive, we use gradient descent to minimize the BC objective: 
\begin{equation}
\mathcal{L}_2(\pi)=\mathbb{E}\big[\text{D}_{\text{KL}}\big(\mu(\cdot|s),\pi(\cdot|o)\big)\big].    
\end{equation}
Similarly, in step 4, we backtrack the guider's 
policy using the same BC loss: 
\begin{equation}
\mathcal{L}_3(\mu)=\mathbb{E}\big[\text{D}_{\text{KL}}\big(\mu(\cdot|s),\pi(\cdot|o)\big)\big].
\end{equation}

A key insight in GPO is that exact backtracking of the guider’s policy is unnecessary—it's sufficient to keep the guider within an imitable region relative to the learner.
The learner may fail to follow the guider either because the guider is inimitable or because the guider learns faster, the latter being common due to inexact gradient updates.
In such cases, aggressive backtracking can be harmful.
Keeping the guider slightly ahead also allows it to collect better trajectories, as discussed in Section \ref{sec:exp4}.
To maintain this balance, we introduce a coefficient $\alpha$ that modulates the guider’s loss as 
\begin{equation}
    \mathcal{L}(\mu)=\mathcal{L}_1(\mu)+\alpha \mathcal{L}_3(\mu),
\end{equation}
where $\alpha$ is adapted based on the distance $L_3(\mu)$ relative to a threshold $d$, using a constant scaling factor $k$: 
\begin{align}\label{eq:alpha}
    \alpha=k\alpha \ \ \text{if}\ \ \mathcal{L}_3(\mu)>kd,\ \ \alpha/k\ \ \text{if}\ \ \mathcal{L}_3(\mu) < d/k.
\end{align}
This scheme is analogous to the KL-penalty adjustment in PPO-penalty \citep{DBLP:journals/corr/SchulmanWDRK17}, where the penalty coefficient adjusts based on the relationship between the KL divergence and a predefined threshold.

Another aspect is compensating for the learner’s policy improvement, as we replace strict backtracking with a KL constraint. 
While it is possible to set a very small $d_{\text{targ}}$, this would inefficiently inflate $\alpha$, hindering the guider's training.
Notably, Proposition \ref{prop} implies that applying GPO with PPO is effectively equivalent to applying PPO directly to the learner. Consequently, we can concurrently train the learner's policy using PPO during the GPO iterations.
As a result, we introduce an additional objective for the learner's policy:
\begin{equation}\label{eq:PPO2}
    \mathcal{L}_4(\pi)=-\mathbb{E}\bigg[\min
    \bigg(\rho^\pi A^\beta(s,a),
    \rho_{clip}^\pi A^\beta(s,a)\bigg)\bigg],
\end{equation}
where $\rho^\pi=\pi(a|o)/\beta(a|s)$.
Considering that the behavioral policy is from the guider, to validate this update, we introduce the following proposition:
\begin{proposition}
    For policy $\pi$, $\mu$, $\beta$ and all states $s$, suppose $\text{D}_\text{TV}(\mu(\cdot|s),\beta(\cdot|s))\lesssim\epsilon/2$, then we have
    \begin{equation*}
        \mathbb{E}_{a\sim\beta}\big[|1-\rho^\pi(s,a)|\big]\lesssim\epsilon+\sqrt{2d_{targ}}.
    \end{equation*}
\end{proposition}
\begin{proof}
    See Appendix \ref{app:proof}.
\end{proof}
The assumption on total variation distance is justified by the PPO update of the guider’s policy (Appendix \ref{app:proof}). This proposition implies that when $d_{\text{targ}}$ is small, the behavioral policy closely matches the learner’s policy, allowing valid sample reuse for learner training.

Finally, we define the merged learner objective for the learner as: 
\begin{equation}
    \mathcal{L}(\pi)=\alpha \mathcal{L}_4(\pi)+\mathcal{L}_2(\pi),
\end{equation}
where the coefficient $\alpha$ from \eqref{eq:alpha} is applied to the RL term. This mechanism compensates when the learner struggles to follow the guider. 
If the learner is able to fully track the guider, $\alpha$ approaches zero, allowing the guider to directly lead the learner to the optimal policy without requiring an additional RL objective. When the learner cannot keep pace, the RL objective aids in the learner’s training.

%\subsection{Refinements of GPO}\label{sec:refine}
\subsection{GPO-Clip}\label{sec:clip}
In this section, we introduce a slightly modified implementation of the GPO framework, which we refer to as GPO-clip. 
The key principle is that an effective guider should remain at the boundary of the learner’s imitable region: if the guider is too far ahead, the learner struggles to follow; if too close, the guider’s ability to provide effective supervision and better trajectory diminishes.
To achieve this balance, the guider should halt updates when it moves too far ahead and avoid backtracking when it is already sufficiently close.

We propose two modifications to the GPO-penalty algorithm introduced in the previous subsection.
First, inspired by PPO-clip, we replace the clip function $\rho_{clip}^\mu$ in \eqref{eq:PPO1} with the following double-clip function:
\begin{equation}\label{eq:clip}
    \rho_{clip}^{\mu,\pi}=\text{clip}\bigg(\text{clip}(\frac{\mu(a|s)}{\pi(a|o)},1-\delta,1+\delta)\cdot\frac{\pi(a|o)}{\beta(a|s)},1-\epsilon,1+\epsilon\bigg).
\end{equation}
This formulation introduces an additional inner clipping step, which halts the guider’s updates under two conditions:
(1) $A^\beta(s,a)>0$ and $\mu(a|s)>\pi(a|o)(1+\delta)$, 
(2) $A^\beta(s,a)<0$ and $\mu(a|s)<\pi(a|o)(1-\delta)$.
Considering that the positive (negative) advantage indicates that $\mu(a|s)$ is set to increase (decrease), the double-clip function prevents further movement away from $\pi$ when $\mu$ is already distant.

It is important to note that, unlike PPO where PPO-clip can completely replace the KL-penalty term, this is not the case in GPO.
In PPO, the ratio $\rho^\pi(s,a)$ starts at 1 at the beginning of each epoch, ensuring that the clipped ratio keeps $\pi$ near the behavioral policy.
In GPO, however, the gap between $\pi(a|s)$ and $\mu(a|o)$ may accumulate over multiple updates if the learner fails to keep up with the guider.
The double-clip function \eqref{eq:clip} alone is insufficient to bring $\pi(a|o)$ back into the $\delta$ region once it has strayed too far.
To address this, we introduce a mask on the backtracking loss, defined as: 
\begin{equation}
   m(s,a)=\mathbb{I}\big(\frac{\mu(a|s)}{\pi(a|o)}\notin(1-\delta,1+\delta)\big) ,
\end{equation}
where $\mathbb{I}$ is the indicator function.
This mask replaces the adaptive coefficient $\alpha$ of GPO-penalty, selectively applying the backtracking penalty only when $\mu(a|o)$ drifts outside the $\delta$ region. Policies that remain close to each other are left unaffected, preventing unnecessary backtracking.

Additionally, given that both the guider and learner are solving the same task, their policies should exhibit structural similarities. To leverage this, we allow the guider and learner to share a single policy network. To distinguish between guider and learner inputs, we define a unified input format:
the input to the guider’s policy is defined as $o_g = [s,o,1]$, where $s$ is the state, $o$ is the partial observation, and the scalar 1 serves as an indicator; the learner’s input is defined as $o_l = [\vec{0},o,0]$, where $\vec{0}$ is a zero vector with the same dimensionality as $s$, indicating that the learner has access only to the partial observation $o$.
This approach is applied to both GPO-penalty and GPO-clip, and the update for the shared policy network with parameters $\theta$ is as follows:
\begin{equation}\label{eq:gpo-p}
\begin{aligned}
    L_{\text{GPO-penalty}}(\theta)=\mathbb{E}\Big[&-\min
    \Big(\rho^{\mu_\theta}A^\beta(o_g,a),
    \rho_{clip}^{\mu_\theta}A^\beta(o_g,a)\Big)+\alpha\text{D}_{\text{KL}}\big(\mu_\theta(\cdot|o_g),\pi_{\hat{\theta}}(\cdot|o_l)\big)\\
    &-\alpha\min
    \Big(\rho^{\pi_\theta}A^\beta(o_l,a),
    \rho_{clip}^{\pi_\theta}A^\beta(o_l,a)\Big)+\text{D}_{\text{KL}}\big(\mu_{\hat{\theta}}(\cdot|o_g),\pi_\theta(\cdot|o_l)\big)
    \Big],
\end{aligned}
\end{equation}
\begin{equation}\label{eq:gpo-c}
\begin{aligned}
    L_{\text{GPO-clip}}(\theta)=\mathbb{E}\Big[&-\min
    \Big(\rho^{\mu_\theta}A^\beta(o_g,a),
    \rho_{clip}^{\mu_\theta,\pi_{\hat{\theta}}}A^\beta(o_g,a)\Big)+m(s,a)\text{D}_{\text{KL}}\big(\mu_\theta(\cdot|o_g),\pi_{\hat{\theta}}(\cdot|o_l)\big)\\
    &-\alpha\min
    \Big(\rho^{\pi_\theta}A^\beta(o_g,a),
    \rho_{clip}^{\pi_\theta}A^\beta(o_g,a)\Big)+\text{D}_{\text{KL}}\big(\mu_{\hat{\theta}}(\cdot|o_g),\pi_\theta(\cdot|o_l)\big)
    \Big],
\end{aligned}
\end{equation}
where $\hat{\theta}$ denotes a stop-gradient operation on the parameters, and $\alpha$ for GPO-clip is a fixed parameter. The detailed algorithms are summarized in Appendix \ref{app:alg}.

\section{Experiments}
In this section, we evaluate the empirical performance of GPO across various domains.
For baselines, we consider two types of approaches for utilizing teacher supervision.
The first type involves training both the teacher and student simultaneously. A summary of their main characteristics is provided in Table \ref{tab:alg}. Among these, \textbf{GPO-naive} refers to GPO-penalty without the RL auxiliary loss. \textbf{PPO-asym} directly trains the learner using PPO, 
\begin{wraptable}{r}{0.55\textwidth}
    \caption{Co-training algorithms.} 
    \label{tab:alg}
    \setlength{\tabcolsep}{2pt}
    \begin{tabular}{c|cccc}
        \toprule
        Algorithm & Train $\mu$ & \thead{Behavioral\\policy} & Train $\pi$ & \thead{Value \\ function} \\
        \midrule
        PPO     & -   & $\pi(a|o_l)$ & PPO     & $V(o_l)$ \\
        PPO-asym  & -   & $\pi(a|o_l)$ & PPO     & $V(o_g)$ \\
        PPO+BC  & PPO & $\mu(a|o_g)$ & BC & $V(o_g)$ \\
        A2D     & PPO & $\pi(a|o_l)$ & BC & $V(o_l)$ \\
        ADVISOR-co & PPO & $\pi(a|o_l)$ & BC+PPO & $V(o_l)$ \\
        GPO-naive  & PPO & $\mu(a|o_g)$ & BC & $V(o_g)$ \\
        GPO-penalty& PPO & $\mu(a|o_g)$ & BC+PPO & $V(o_g)$ \\
        GPO-clip   & PPO & $\mu(a|o_g)$ & BC+PPO & $V(o_g)$ \\
        GPO-ablation & PPO & $\mu(a|o_g)$ & PPO & $V(o_g)$ \\
        \bottomrule
    \end{tabular}
\end{wraptable}
with the learner’s value function receiving $o_g$ as input. \textbf{PPO+BC} trains the teacher with PPO, while the learner is trained via direct BC from the teacher. \textbf{ADVISOR-co} is a modification of ADVISOR \citep{10.5555/3540261.3541724}, and \textbf{A2D} is based on the work by \citep{Warrington2020RobustAL}.
The second type involves training the teacher first, followed by the application of TSL methods. These include \textbf{DAgger} \citep{ross2011reduction}, \textbf{PPO+BC-t}, \textbf{ADVISOR}, \textbf{ELF} \citep{Walsman2023ImpossiblyGE}, and \textbf{ELF-asym}, where ELF is a policy distillation method that utilizes reward shaping to provide supervision to the student, and ELF-asym is a variant that uses an asymmetric value function.
Further details about these algorithms can be found in Appendix \ref{app:base}.

\subsection{Didactic Tasks}
\label{sec:exp1}

\begin{figure}[ht]
\centering
% === 右侧图像 ===
\begin{minipage}[ht]{0.7\textwidth}
    \centering
    \includegraphics[width=\linewidth]{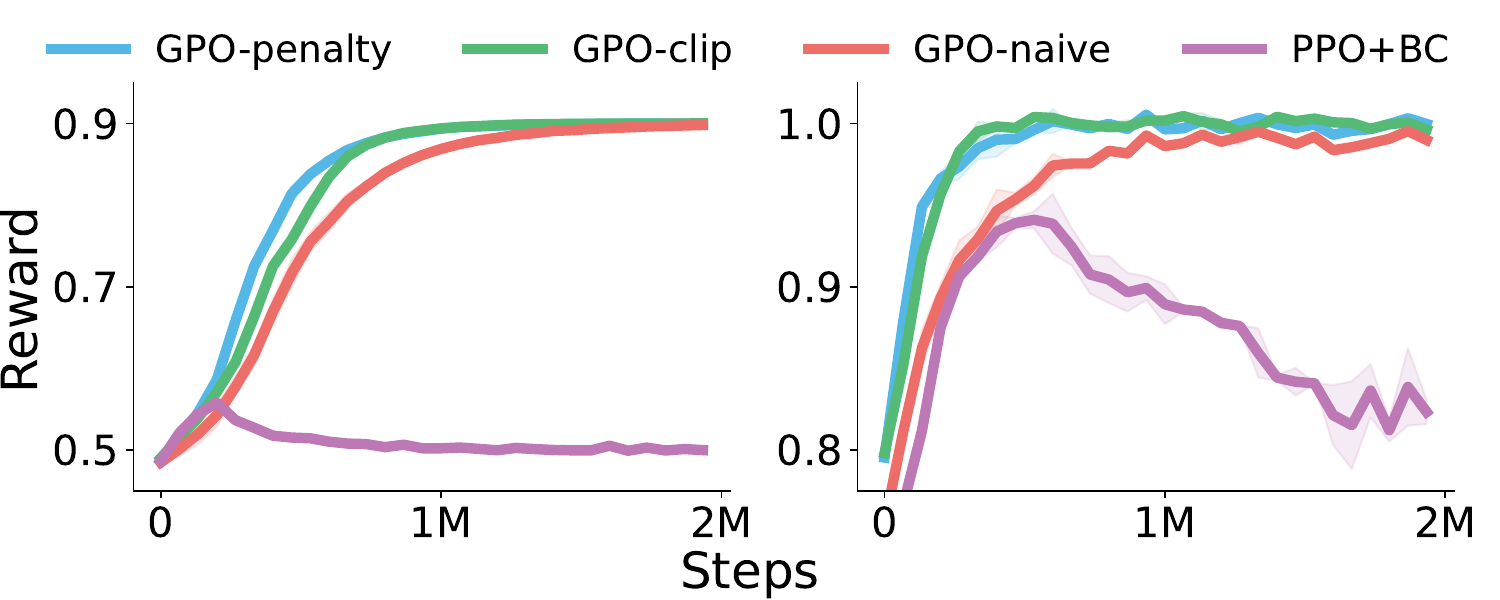}
    \captionof{figure}{Results on TigerDoor (left) and TigerDoor-alt (right).}
    \label{fig:dida}
\end{minipage}
\end{figure}

We begin by evaluating our algorithm on two didactic problems introduced in Section \ref{sec:dida}.
As shown in Fig. \ref{fig:dida}, direct cloning of the guider’s policy converges to a suboptimal solution, as expected.
In contrast, all variants of GPO achieve optimal performance on these tasks.
Although applying RL directly to the learner easily leads to optimal solutions, it is important to note that GPO-naive achieves optimality purely through supervised learning. 
This result verifies the optimality guarantee of the GPO framework described in Proposition \ref{prop}, suggesting that a guider constrained within the learner's imitable region can provide effective supervision, even with asymmetric information.
Moreover, comparing GPO-naive to GPO-penalty and GPO-clip reveals that the introduction of direct RL training for the learner accelerates learning.
% Moreover, as shown in Fig. \ref{fig:dida}(c), the optimality of GPO-naive is robust to variations in the KL-threshold, offering flexibility to adjust the distance between the guider and learner across different tasks.

\subsection{Continuous Control Tasks in Brax}
\label{sec:exp2}

In this subsection, we present the results of our algorithms and baselines on several continuous control tasks in the Brax domain.
To transform these tasks into a POMDP setting, we remove the velocity information of all joints, and add varying levels of noise to the observations. The guider has access to full, noiseless information, while the learner operates with partial and noisy inputs.
For more details, please refer to Appendix \ref{app:exp}.

\begin{figure*}[t]
\centering
\includegraphics[width=0.99\linewidth]{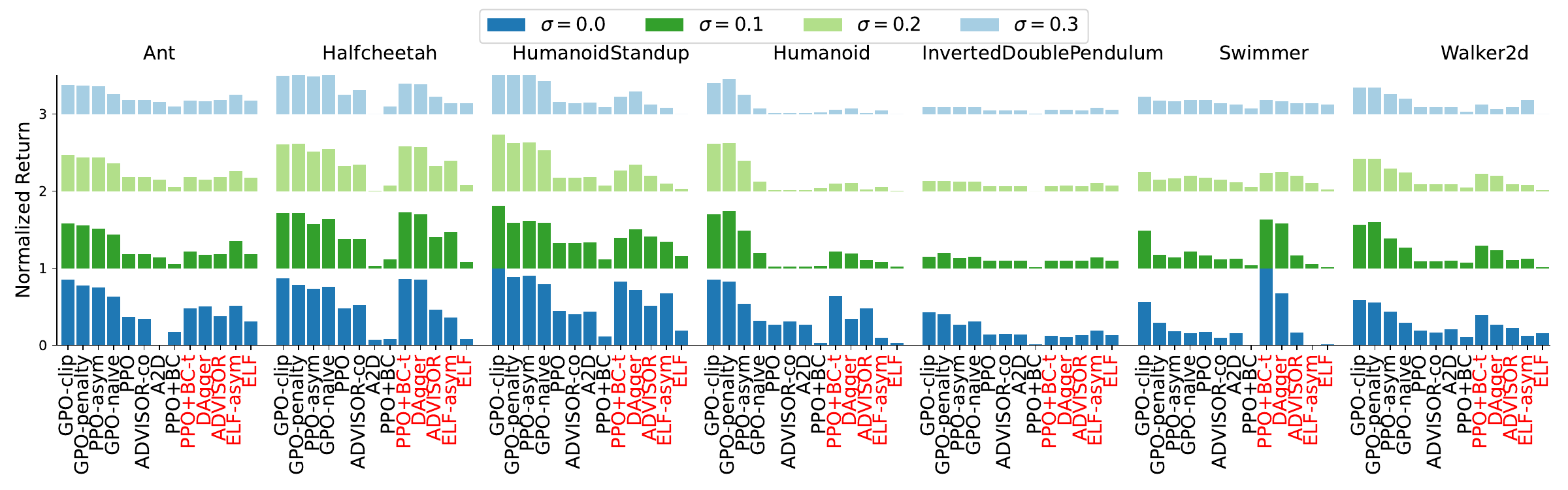}
  \caption{Comparison of GPO and baselines on the Brax domain, where $\sigma$ represents the scale of Gaussian noise added to the observations. The performance on each task is normalized to $[0,1]$ using the performance of the corresponding pre-trained teacher as a reference. Algorithms highlighted in red rely on a separately pre-trained privileged teacher.
  }
\label{fig:mujoco}
\end{figure*}

The results are shown in Fig. \ref{fig:mujoco}, where the performance hierarchy is generally: GPO-clip $>$ GPO-penalty $>$ PPO-asym $>$ GPO-naive $>$ other baselines. 
It is important to note that, even without factoring in the cost of training the teacher (which is comparable to training GPO from scratch), methods that rely on a pre-trained privileged teacher perform well only in the \textit{Halfcheetah} and \textit{Swimmer} tasks. 
Furthermore, the performance of these methods declines rapidly as the noise scale increases. 
This occurs because, when the pre-trained teacher becomes too skilled for the student, it provides little to no useful supervision, and may even have a negative impact on learning.

For co-training approaches, we have the following observations:
First, the superior performance of GPO-clip and GPO-penalty compared to the base algorithm PPO shows that this framework can effectively utilize additional information during training to facilitate the learner training.
Second, comparing GPO-naive to GPO-penalty and GPO-clip, we see that introducing RL training for the learner improves performance.
Third, the comparison between PPO+BC and GPO-naive highlights the necessity of backtracking. 
If the guider is not constrained to the learner, the guider’s supervision may negatively influence the performance.
Last, other baselines such as ADVISOR failed to utilize the privileged teacher such that it degenerates into pure PPO.

In summary, our method consistently outperforms the baselines, demonstrating its effectiveness in solving noisy and partially observable continuous control tasks. 
Additional experiments including L2T-RL \citep{wu2024learnteachimprovesample}, TGRL \citep{shenfeld2023tgrl} and RMA \citep{kumar2021rma} are provided in Appendix \ref{app:teacher}.

\subsection{Memory-based Tasks in POPGym}\label{sec:exp3}

Since using memory models to deal with POMDP is a common practice, we evaluate GPO in POPGym to show whether the algorithm can effectively address memory-based tasks.
The tasks include card and board games where agents must recall previous observations to extract useful information for decision-making. 
For these tasks, the guider’s observation is designed to include the critical information needed to remember, theoretically minimizing the imitation gap as long as the memory model can store the necessary information.
Although in practice, memory models struggle to retain all information, especially in complex tasks, this setup allows us to use a larger KL threshold or clipping parameter, enabling the guider to explore further and provide more valuable supervision.
Further details on the experimental settings are provided in Appendix \ref{app:exp}.

\begin{figure}[ht]
\centering
\includegraphics[width=0.99\linewidth]{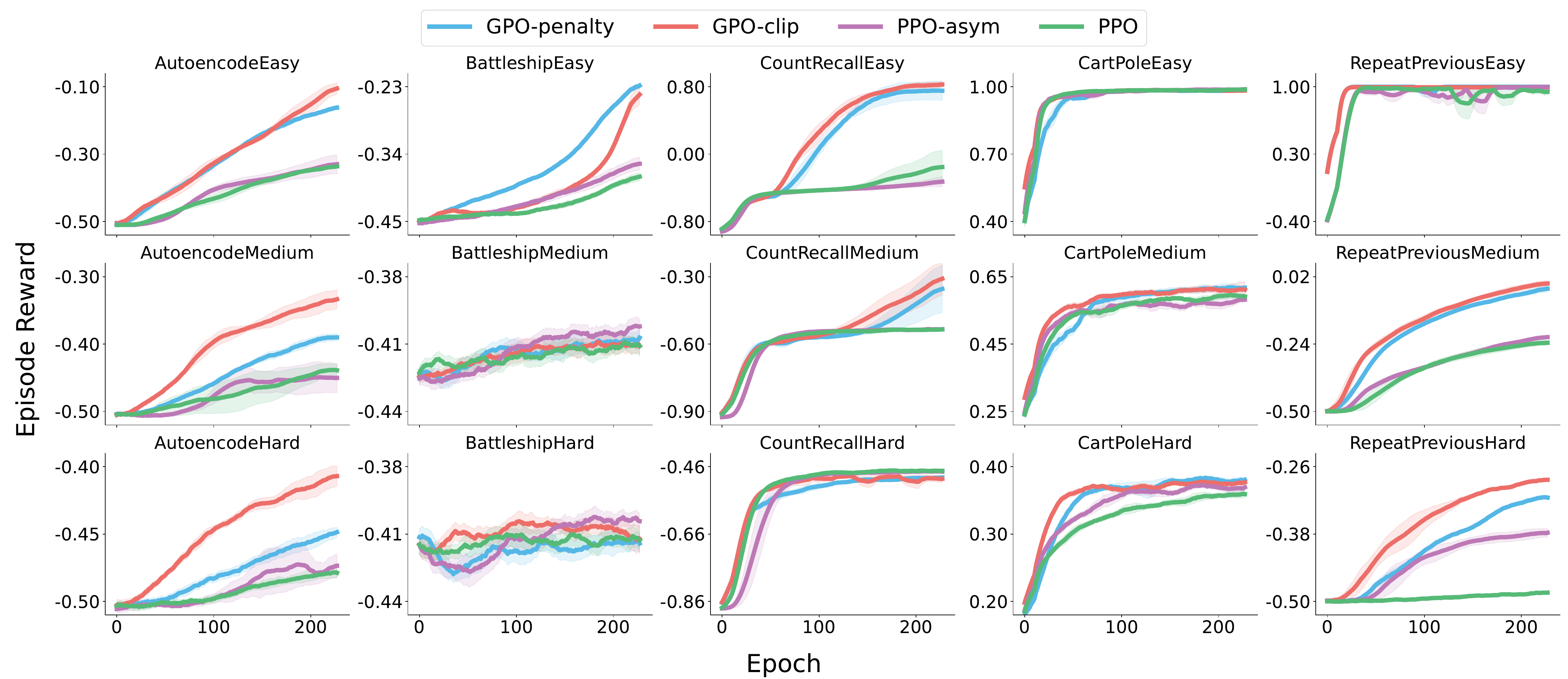}
\caption{The results of GPO-clip, GPO-penalty, PPO-asym, and PPO on 15 POPGym tasks.}\label{fig:pop}
\end{figure}

Fig. \ref{fig:pop} shows the results on 15 POPGym tasks, where we compare GPO-penalty and GPO-clip to PPO-asym and PPO.
The general conclusion mirrors the results from the previous subsection, where GPO-clip typically outperforms GPO-penalty, followed by PPO-asym and PPO. Key insights include:
First, the superior performance of GPO-penalty indicates that the ability of the guider to explore further without diverging too much from the learner proves valuable in these memory-based tasks.
Second, while PPO-asym outperforms PPO, its performance improvement is less pronounced here than in the Brax domain, suggesting that asymmetric value function may not be very helpful for memory tasks.
Third, although neither GPO-penalty nor GPO-clip exhibits superior performance in tasks like \textit{BattleshipMedium} and \textit{CountRecallHard}, this is due to the fact that we use the same parameter across all tasks, and performance could be improved as we show in the next section.

Overall, our methods demonstrate strong performance across the majority of tasks, providing an effective solution for memory-based problems.

\subsection{Ablations and Discussions}\label{sec:exp4}

In this section, we dive deeper into GPO's performance through ablations and further discussions.

\begin{figure}[ht]
\begin{minipage}[t]{0.48\textwidth}
  \begin{subfigure}{0.96\linewidth}
  \setlength{\abovecaptionskip}{5pt}
        \centering
    \includegraphics[width=1\linewidth]{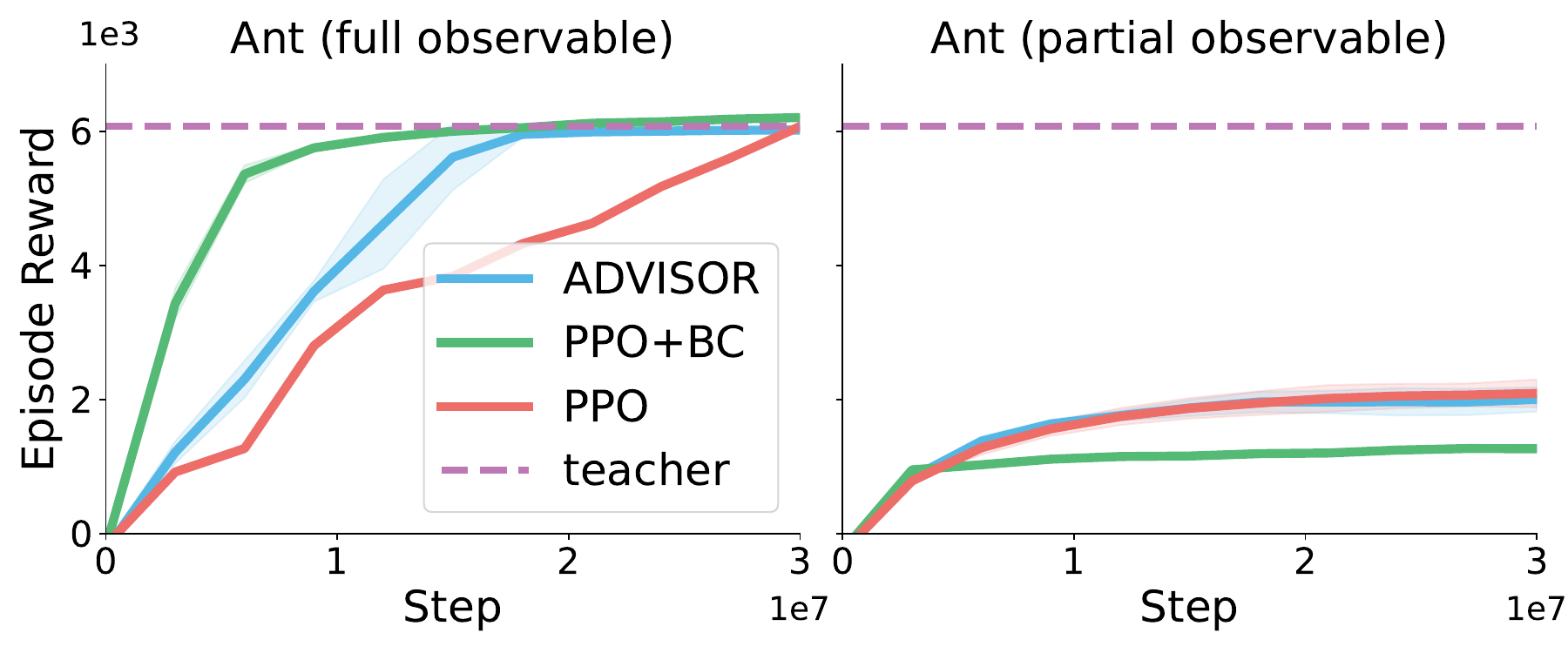}
  \end{subfigure}
  \caption{ADVISOR and PPO+BC with a pre-trained teacher.}\label{fig:teacher}
\end{minipage}
\hspace{2mm}
\begin{minipage}[t]{0.49\textwidth}
\begin{subfigure}{0.49\linewidth}
  \setlength{\abovecaptionskip}{5pt}
        \centering
    \includegraphics[width=1\linewidth]{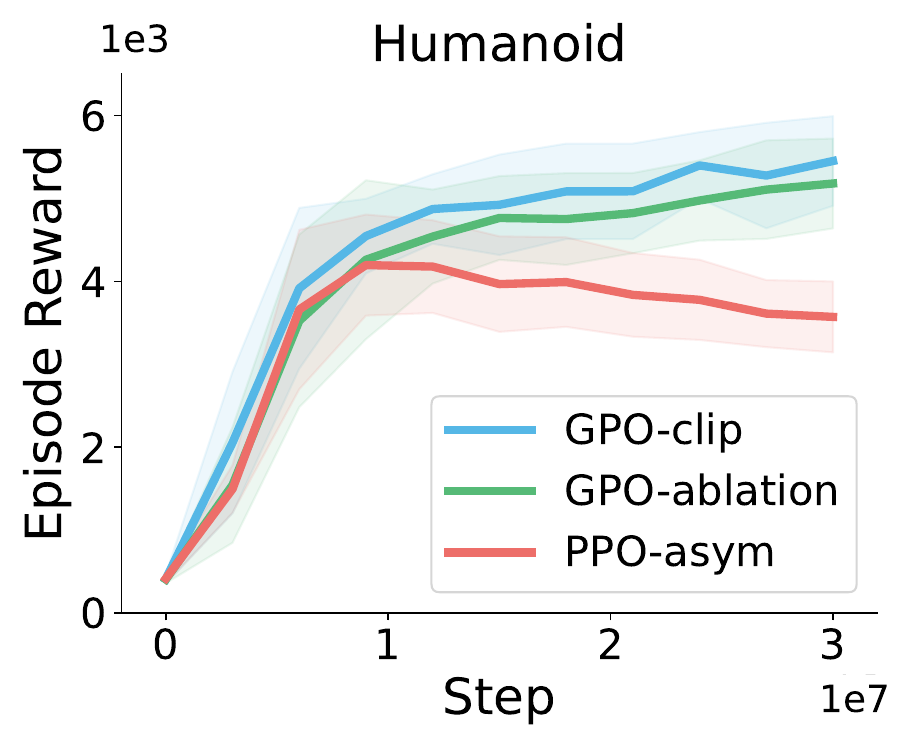}
  \end{subfigure}
  \begin{subfigure}{0.49\linewidth}
  \setlength{\abovecaptionskip}{5pt}
        \centering
    \includegraphics[width=1\linewidth]{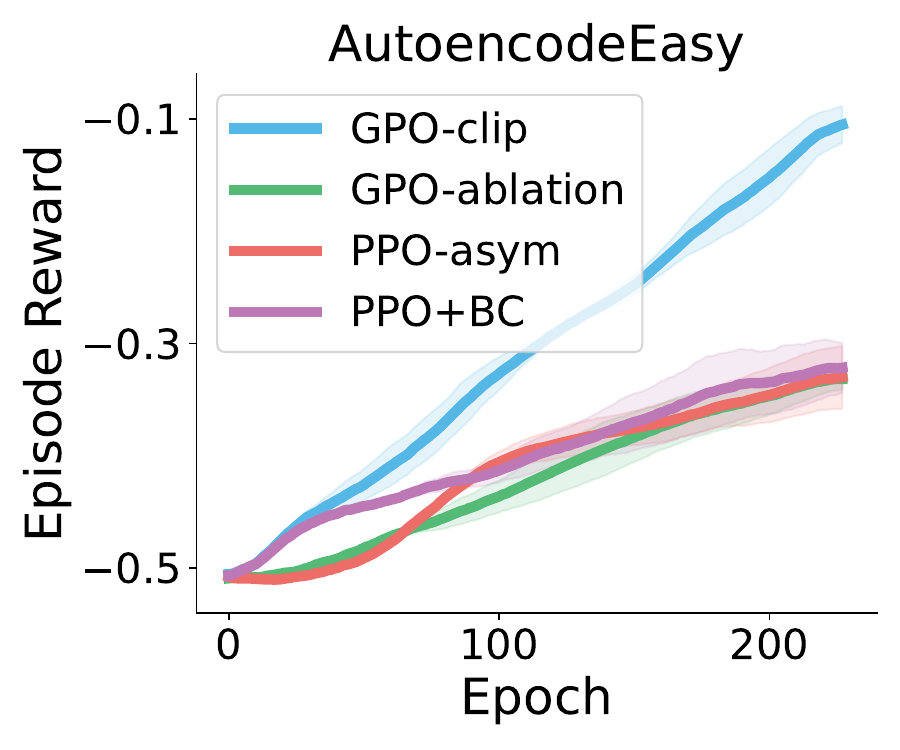}
  \end{subfigure}
  \caption{Ablation studies.}\label{fig:abl}
\end{minipage}
\end{figure}

\textbf{Why does training a teacher first and applying TSL methods often fail?}
A representative example is the TigerDoor problem, where a pre-trained teacher provides minimal to no effective supervision for the student.
Recent TSL approaches, such as ADVISOR and TGRL, address the challenge of an overly optimal teacher by reverting to pure RL, thereby bypassing uninformative or misleading supervision.
As shown in Fig. \ref{fig:teacher}, although ADVISOR and PPO+BC perform well in the fully observable \textit{Ant} task where the teacher is trained, it degenerates into PPO in the partially observable \textit{Ant} task since the teacher is found inimitable.

\textbf{Why do GPO outperform other baselines?}
We attribute the superior performance to two factors: effective RL training of the learner, and effective supervision from the guider.
The benefit of RL training is shown in Fig. \ref{fig:abl}(left), where GPO-ablation (GPO-penalty without supervision, as described in Table \ref{tab:alg}) outperforms PPO-asym on the \textit{Humanoid} task.
Although both use similar objectives, GPO-ablation uses data collected by the guider, indicating that a better behavior policy improves learning efficiency.
The effectiveness of the supervision comes from the guider being constrained to the imitable region while still learning rapidly. 
In Fig. \ref{fig:abl}(right), with the learner trained purely by supervision (GPO-clip with RL disabled), GPO-clip outperforms GPO-ablation, PPO+BC, and PPO-asym.
This shows that in memory-intensive tasks, supervision is more beneficial than RL.
Since PPO+BC performs poorly in noisy tasks in Section \ref{sec:exp2} but comparably to PPO-asym here, we can also infer that supervision plays a particularly important role in these tasks.
Moreover, GPO-clip’s strong performance over PPO+BC—despite both using pure supervision—highlights the importance of constraining the guider to a policy the learner can follow.

\begin{figure}[ht]
\centering
\begin{subfigure}{0.49\columnwidth}
  \setlength{\abovecaptionskip}{5pt}
        \centering
\includegraphics[width=0.95\linewidth]{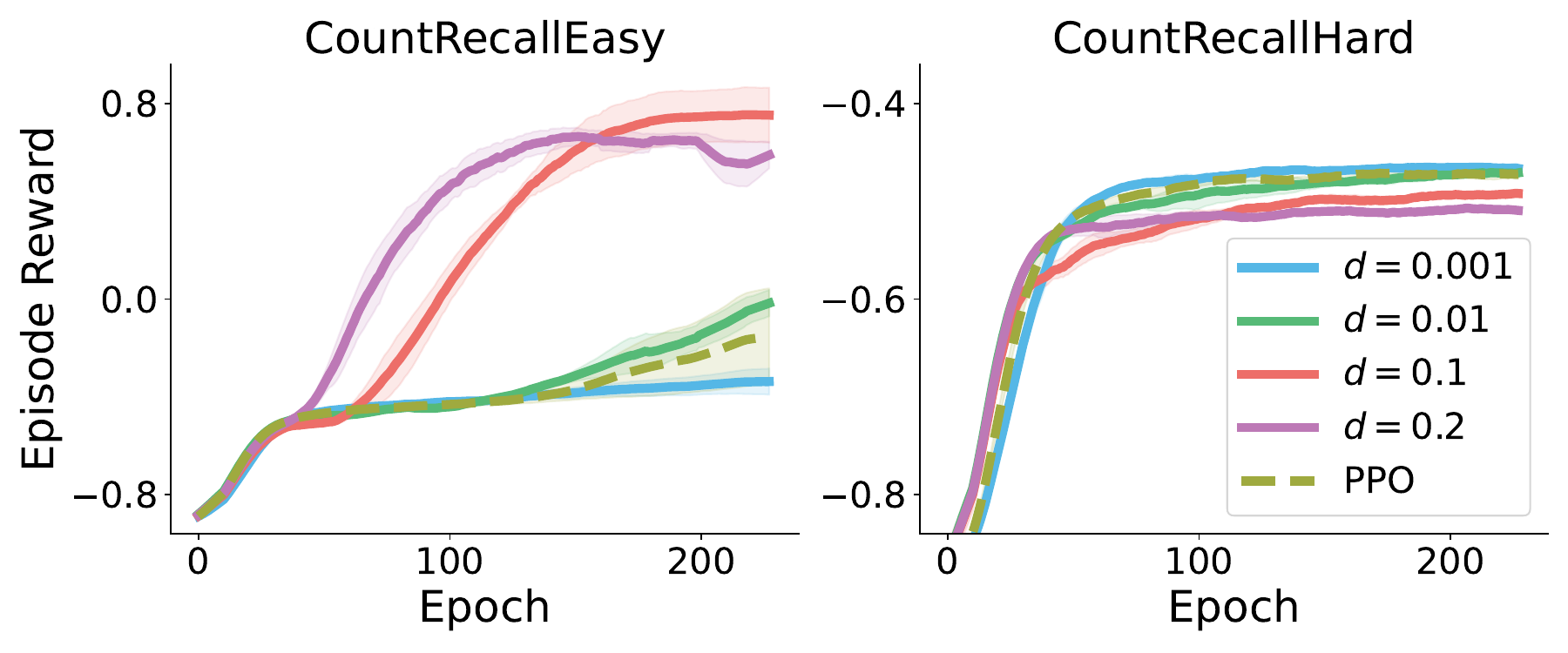}
\caption{GPO-penalty with different KL-threshold.}
  \end{subfigure}
\begin{subfigure}{0.49\columnwidth}
  \setlength{\abovecaptionskip}{5pt}
        \centering
\includegraphics[width=0.95\linewidth]{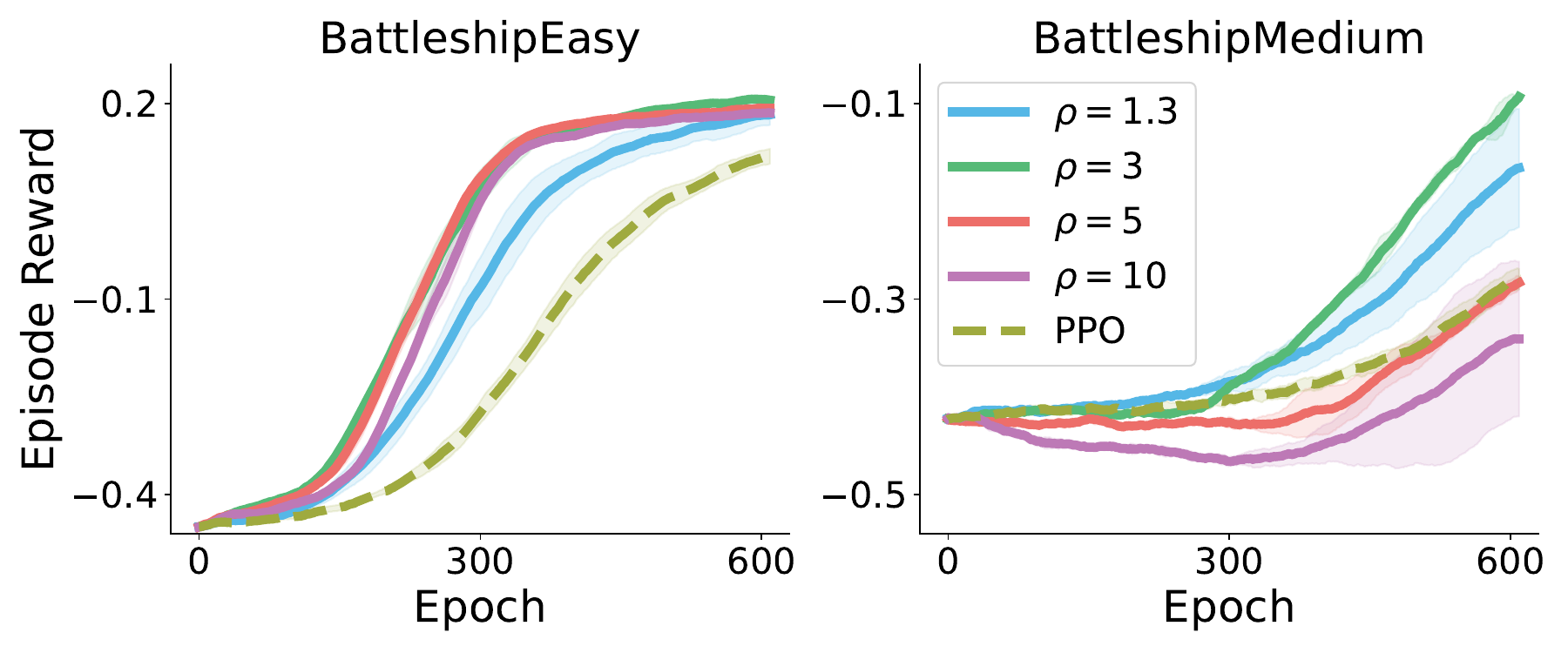}
\caption{GPO-clip with different clip parameters.}
  \end{subfigure}
  \caption{The results of GPO-penalty and GPO-clip with different hyperparameters. The clip parameter $\rho$ is defined in Appendix \ref{sec:hyper}.}\label{fig:hyper}
\end{figure}

\textbf{When does GPO fail}? 
GPO can fail when the guider learns too slowly, often due to inadequate information.
Another failure mode arises from poorly tuned KL thresholds (clip parameters).
For instance, in the \textit{CountRecallHard} task from POPGym, both GPO variants underperform compared to PPO and PPO-asym. 
As shown in Fig. \ref{fig:hyper}, larger KL thresholds help in simple tasks like \textit{CountRecallEasy} and \textit{BattleshipEasy}, but hurt performance in harder ones like \textit{CountRecallHard} and \textit{BattleshipMedium}. 
This is because challenging tasks strain memory models like GRU—when GRU fails to retain key information, the learner cannot follow the guider.
In such cases, a large KL threshold pushes the guider beyond the learner’s reachable region, causing an unrecoverable imitation gap.

\section{Conclusion and Future work}
In this paper, we introduce GPO, a method designed to leverage additional information in POMDPs during training. Our experimental results demonstrate that the proposed algorithm effectively addresses noisy and memory-based partially observable tasks, offering a novel approach to utilizing auxiliary information for more efficient learning.
Future work could explore extending guided policy optimization to the multi-agent setting, where agents often have access to global information during training but are constrained to local observations during execution. 

\subsubsection*{Acknowledgments}
This work was supported in part by the National Natural Science Foundation of China [grant numbers U22A2062, U23B2037, 12272008, 62450001, 62476008].

\bibliography{iclr2026_conference}

\begin{thebibliography}{61}
\providecommand{\natexlab}[1]{#1}
\providecommand{\url}[1]{\texttt{#1}}
\expandafter\ifx\csname urlstyle\endcsname\relax
  \providecommand{\doi}[1]{doi: #1}\else
  \providecommand{\doi}{doi: \begingroup \urlstyle{rm}\Url}\fi

\bibitem[Andrychowicz et~al.(2020)Andrychowicz, Baker, Chociej, Józefowicz, McGrew, Pachocki, Petron, Plappert, Powell, Ray, Schneider, Sidor, Tobin, Welinder, Weng, and Zaremba]{doi:10.1177/0278364919887447}
OpenAI:~Marcin Andrychowicz, Bowen Baker, Maciek Chociej, Rafal Józefowicz, Bob McGrew, Jakub Pachocki, Arthur Petron, Matthias Plappert, Glenn Powell, Alex Ray, Jonas Schneider, Szymon Sidor, Josh Tobin, Peter Welinder, Lilian Weng, and Wojciech Zaremba.
\newblock Learning dexterous in-hand manipulation.
\newblock \emph{The International Journal of Robotics Research}, 39\penalty0 (1):\penalty0 3--20, 2020.
\newblock \doi{10.1177/0278364919887447}.

\bibitem[Baisero \& Amato(2021)Baisero and Amato]{baisero2021unbiased}
Andrea Baisero and Christopher Amato.
\newblock Unbiased asymmetric reinforcement learning under partial observability.
\newblock \emph{arXiv preprint arXiv:2105.11674}, 2021.

\bibitem[Bansal et~al.(2018)Bansal, Krizhevsky, and Ogale]{bansal2018chauffeurnet}
Mayank Bansal, Alex Krizhevsky, and Abhijit Ogale.
\newblock Chauffeurnet: Learning to drive by imitating the best and synthesizing the worst.
\newblock \emph{arXiv preprint arXiv:1812.03079}, 2018.

\bibitem[Bhardwaj et~al.(2017)Bhardwaj, Choudhury, and Scherer]{bhardwaj2017learning}
Mohak Bhardwaj, Sanjiban Choudhury, and Sebastian Scherer.
\newblock Learning heuristic search via imitation.
\newblock In \emph{Conference on Robot Learning}, pp.\  271--280. PMLR, 2017.

\bibitem[Cai et~al.(2024)Cai, Liu, Oikonomou, and Zhang]{NEURIPS2024_74d188c5}
Yang Cai, Xiangyu Liu, Argyris Oikonomou, and Kaiqing Zhang.
\newblock Provable partially observable reinforcement learning with privileged information.
\newblock In A.~Globerson, L.~Mackey, D.~Belgrave, A.~Fan, U.~Paquet, J.~Tomczak, and C.~Zhang (eds.), \emph{Advances in Neural Information Processing Systems}, volume~37, pp.\  63790--63857. Curran Associates, Inc., 2024.
\newblock URL \url{https://proceedings.neurips.cc/paper_files/paper/2024/file/74d188c51d97fcfbc0269f584d6a53b7-Paper-Conference.pdf}.

\bibitem[Chang et~al.(2015)Chang, Krishnamurthy, Daum\'e, and Langford]{pmlr-v37-changb15}
Kai-Wei Chang, Akshay Krishnamurthy, Hal Daum\'e, III, and John Langford.
\newblock Learning to search better than your teacher.
\newblock In Francis Bach and David Blei (eds.), \emph{Proceedings of the 32nd International Conference on Machine Learning}, volume~37 of \emph{Proceedings of Machine Learning Research}, pp.\  2058--2066, Lille, France, 07--09 Jul 2015. PMLR.
\newblock URL \url{https://proceedings.mlr.press/v37/changb15.html}.

\bibitem[Chen et~al.(2022)Chen, Xu, and Agrawal]{pmlr-v164-chen22a}
Tao Chen, Jie Xu, and Pulkit Agrawal.
\newblock A system for general in-hand object re-orientation.
\newblock In Aleksandra Faust, David Hsu, and Gerhard Neumann (eds.), \emph{Proceedings of the 5th Conference on Robot Learning}, volume 164 of \emph{Proceedings of Machine Learning Research}, pp.\  297--307. PMLR, 08--11 Nov 2022.
\newblock URL \url{https://proceedings.mlr.press/v164/chen22a.html}.

\bibitem[Choudhury et~al.(2017)Choudhury, Kapoor, Ranade, Scherer, and Dey]{choudhury2017adaptive}
Sanjiban Choudhury, Ashish Kapoor, Gireeja Ranade, Sebastian Scherer, and Debadeepta Dey.
\newblock Adaptive information gathering via imitation learning.
\newblock \emph{arXiv preprint arXiv:1705.07834}, 2017.

\bibitem[Czarnecki et~al.(2019)Czarnecki, Pascanu, Osindero, Jayakumar, Swirszcz, and Jaderberg]{pmlr-v89-czarnecki19a}
Wojciech~M. Czarnecki, Razvan Pascanu, Simon Osindero, Siddhant Jayakumar, Grzegorz Swirszcz, and Max Jaderberg.
\newblock Distilling policy distillation.
\newblock In Kamalika Chaudhuri and Masashi Sugiyama (eds.), \emph{Proceedings of the Twenty-Second International Conference on Artificial Intelligence and Statistics}, volume~89 of \emph{Proceedings of Machine Learning Research}, pp.\  1331--1340. PMLR, 16--18 Apr 2019.
\newblock URL \url{https://proceedings.mlr.press/v89/czarnecki19a.html}.

\bibitem[De~Haan et~al.(2019)De~Haan, Jayaraman, and Levine]{de2019causal}
Pim De~Haan, Dinesh Jayaraman, and Sergey Levine.
\newblock Causal confusion in imitation learning.
\newblock \emph{Advances in neural information processing systems}, 32, 2019.

\bibitem[Efroni et~al.(2022)Efroni, Jin, Krishnamurthy, and Miryoosefi]{efroni2022provable}
Yonathan Efroni, Chi Jin, Akshay Krishnamurthy, and Sobhan Miryoosefi.
\newblock Provable reinforcement learning with a short-term memory.
\newblock In \emph{International Conference on Machine Learning}, pp.\  5832--5850. PMLR, 2022.

\bibitem[Foerster et~al.(2018)Foerster, Farquhar, Afouras, Nardelli, and Whiteson]{foerster2018counterfactual}
Jakob Foerster, Gregory Farquhar, Triantafyllos Afouras, Nantas Nardelli, and Shimon Whiteson.
\newblock Counterfactual multi-agent policy gradients.
\newblock In \emph{Proceedings of the AAAI conference on artificial intelligence}, volume~32, 2018.

\bibitem[Freeman et~al.(2021)Freeman, Frey, Raichuk, Girgin, Mordatch, and Bachem]{brax2021github}
C~Daniel Freeman, Erik Frey, Anton Raichuk, Sertan Girgin, Igor Mordatch, and Olivier Bachem.
\newblock Brax--a differentiable physics engine for large scale rigid body simulation.
\newblock \emph{arXiv preprint arXiv:2106.13281}, 2021.

\bibitem[Haarnoja et~al.(2018)Haarnoja, Zhou, Abbeel, and Levine]{haarnoja2018soft}
Tuomas Haarnoja, Aurick Zhou, Pieter Abbeel, and Sergey Levine.
\newblock Soft actor-critic: Off-policy maximum entropy deep reinforcement learning with a stochastic actor.
\newblock In \emph{International conference on machine learning}, pp.\  1861--1870. PMLR, 2018.

\bibitem[Hussein et~al.(2017)Hussein, Gaber, Elyan, and Jayne]{hussein2017imitation}
Ahmed Hussein, Mohamed~Medhat Gaber, Eyad Elyan, and Chrisina Jayne.
\newblock Imitation learning: A survey of learning methods.
\newblock \emph{ACM Computing Surveys (CSUR)}, 50\penalty0 (2):\penalty0 1--35, 2017.

\bibitem[Igl et~al.(2018)Igl, Zintgraf, Le, Wood, and Whiteson]{igl2018deep}
Maximilian Igl, Luisa Zintgraf, Tuan~Anh Le, Frank Wood, and Shimon Whiteson.
\newblock Deep variational reinforcement learning for pomdps.
\newblock In \emph{International conference on machine learning}, pp.\  2117--2126. PMLR, 2018.

\bibitem[Kaelbling et~al.(1998)Kaelbling, Littman, and Cassandra]{KAELBLING199899}
Leslie~Pack Kaelbling, Michael~L. Littman, and Anthony~R. Cassandra.
\newblock Planning and acting in partially observable stochastic domains.
\newblock \emph{Artificial Intelligence}, 101\penalty0 (1):\penalty0 99--134, 1998.
\newblock ISSN 0004-3702.
\newblock \doi{https://doi.org/10.1016/S0004-3702(98)00023-X}.
\newblock URL \url{https://www.sciencedirect.com/science/article/pii/S000437029800023X}.

\bibitem[Kumar et~al.(2021)Kumar, Fu, Pathak, and Malik]{kumar2021rma}
Ashish Kumar, Zipeng Fu, Deepak Pathak, and Jitendra Malik.
\newblock Rma: Rapid motor adaptation for legged robots.
\newblock \emph{arXiv preprint arXiv:2107.04034}, 2021.

\bibitem[Lambert et~al.(2018)Lambert, Sener, and Savarese]{lambert2018deep}
John Lambert, Ozan Sener, and Silvio Savarese.
\newblock Deep learning under privileged information using heteroscedastic dropout.
\newblock In \emph{Proceedings of the IEEE Conference on Computer Vision and Pattern Recognition}, pp.\  8886--8895, 2018.

\bibitem[Lee et~al.(2023)Lee, Agarwal, Dann, and Zhang]{lee2023learning}
Jonathan Lee, Alekh Agarwal, Christoph Dann, and Tong Zhang.
\newblock Learning in pomdps is sample-efficient with hindsight observability.
\newblock In \emph{International Conference on Machine Learning}, pp.\  18733--18773. PMLR, 2023.

\bibitem[Lee et~al.(2020{\natexlab{a}})Lee, Hwangbo, Wellhausen, Koltun, and Hutter]{doi:10.1126/scirobotics.abc5986}
Joonho Lee, Jemin Hwangbo, Lorenz Wellhausen, Vladlen Koltun, and Marco Hutter.
\newblock Learning quadrupedal locomotion over challenging terrain.
\newblock \emph{Science Robotics}, 5\penalty0 (47):\penalty0 eabc5986, 2020{\natexlab{a}}.
\newblock \doi{10.1126/scirobotics.abc5986}.
\newblock URL \url{https://www.science.org/doi/abs/10.1126/scirobotics.abc5986}.

\bibitem[Lee et~al.(2020{\natexlab{b}})Lee, Hwangbo, Wellhausen, Koltun, and Hutter]{lee2020learning}
Joonho Lee, Jemin Hwangbo, Lorenz Wellhausen, Vladlen Koltun, and Marco Hutter.
\newblock Learning quadrupedal locomotion over challenging terrain.
\newblock \emph{Science robotics}, 5\penalty0 (47):\penalty0 eabc5986, 2020{\natexlab{b}}.

\bibitem[Levine \& Koltun(2013)Levine and Koltun]{levine2013guided}
Sergey Levine and Vladlen Koltun.
\newblock Guided policy search.
\newblock In \emph{International conference on machine learning}, pp.\  1--9. PMLR, 2013.

\bibitem[Littman et~al.(1995)Littman, Cassandra, and Kaelbling]{10.5555/3091622.3091666}
Michael~L. Littman, Anthony~R. Cassandra, and Leslie~Pack Kaelbling.
\newblock Learning policies for partially observable environments: scaling up.
\newblock In \emph{Proceedings of the Twelfth International Conference on International Conference on Machine Learning}, ICML'95, pp.\  362–370, San Francisco, CA, USA, 1995. Morgan Kaufmann Publishers Inc.
\newblock ISBN 1558603778.

\bibitem[Liu et~al.(2022)Liu, Chung, Szepesv{\'a}ri, and Jin]{liu2022partially}
Qinghua Liu, Alan Chung, Csaba Szepesv{\'a}ri, and Chi Jin.
\newblock When is partially observable reinforcement learning not scary?
\newblock In \emph{Conference on Learning Theory}, pp.\  5175--5220. PMLR, 2022.

\bibitem[Lowe et~al.(2017)Lowe, Wu, Tamar, Harb, Pieter~Abbeel, and Mordatch]{lowe2017multi}
Ryan Lowe, Yi~I Wu, Aviv Tamar, Jean Harb, OpenAI Pieter~Abbeel, and Igor Mordatch.
\newblock Multi-agent actor-critic for mixed cooperative-competitive environments.
\newblock \emph{Advances in neural information processing systems}, 30, 2017.

\bibitem[Lu et~al.(2023)Lu, Schroecker, Gu, Parisotto, Foerster, Singh, and Behbahani]{lu2023structured}
Chris Lu, Yannick Schroecker, Albert Gu, Emilio Parisotto, Jakob Foerster, Satinder Singh, and Feryal Behbahani.
\newblock Structured state space models for in-context reinforcement learning.
\newblock \emph{arXiv preprint arXiv:2303.03982}, 2023.

\bibitem[Madani et~al.(1999)Madani, Hanks, and Condon]{madani1999undecidability}
Omid Madani, Steve Hanks, and Anne Condon.
\newblock On the undecidability of probabilistic planning and infinite-horizon partially observable markov decision problems.
\newblock \emph{Aaai/iaai}, 10\penalty0 (315149.315395), 1999.

\bibitem[Meng et~al.(2021)Meng, Gorbet, and Kuli{\'c}]{meng2021memory}
Lingheng Meng, Rob Gorbet, and Dana Kuli{\'c}.
\newblock Memory-based deep reinforcement learning for pomdps.
\newblock In \emph{2021 IEEE/RSJ international conference on intelligent robots and systems (IROS)}, pp.\  5619--5626. IEEE, 2021.

\bibitem[Montgomery \& Levine(2016)Montgomery and Levine]{NIPS2016_a00e5eb0}
William~H Montgomery and Sergey Levine.
\newblock Guided policy search via approximate mirror descent.
\newblock In D.~Lee, M.~Sugiyama, U.~Luxburg, I.~Guyon, and R.~Garnett (eds.), \emph{Advances in Neural Information Processing Systems}, volume~29. Curran Associates, Inc., 2016.
\newblock URL \url{https://proceedings.neurips.cc/paper_files/paper/2016/file/a00e5eb0973d24649a4a920fc53d9564-Paper.pdf}.

\bibitem[Morad et~al.(2023)Morad, Kortvelesy, Bettini, Liwicki, and Prorok]{morad2023popgym}
Steven Morad, Ryan Kortvelesy, Matteo Bettini, Stephan Liwicki, and Amanda Prorok.
\newblock {POPG}ym: Benchmarking partially observable reinforcement learning.
\newblock In \emph{The Eleventh International Conference on Learning Representations}, 2023.
\newblock URL \url{https://openreview.net/forum?id=chDrutUTs0K}.

\bibitem[Ng et~al.(1999)Ng, Harada, and Russell]{ng1999policy}
Andrew~Y Ng, Daishi Harada, and Stuart Russell.
\newblock Policy invariance under reward transformations: Theory and application to reward shaping.
\newblock In \emph{Icml}, volume~99, pp.\  278--287, 1999.

\bibitem[Nguyen et~al.(2022)Nguyen, Baisero, Wang, Amato, and Platt]{nguyen2022leveraging}
Hai Nguyen, Andrea Baisero, Dian Wang, Christopher Amato, and Robert Platt.
\newblock Leveraging fully observable policies for learning under partial observability.
\newblock \emph{arXiv preprint arXiv:2211.01991}, 2022.

\bibitem[Nguyen et~al.(2023)Nguyen, Baisero, Wang, Amato, and Platt]{pmlr-v205-nguyen23a}
Hai~Huu Nguyen, Andrea Baisero, Dian Wang, Christopher Amato, and Robert Platt.
\newblock Leveraging fully observable policies for learning under partial observability.
\newblock In Karen Liu, Dana Kulic, and Jeff Ichnowski (eds.), \emph{Proceedings of The 6th Conference on Robot Learning}, volume 205 of \emph{Proceedings of Machine Learning Research}, pp.\  1673--1683. PMLR, 14--18 Dec 2023.
\newblock URL \url{https://proceedings.mlr.press/v205/nguyen23a.html}.

\bibitem[Pinto et~al.(2018)Pinto, Andrychowicz, Welinder, Zaremba, and Abbeel]{pinto2018asymmetric}
Lerrel Pinto, Marcin Andrychowicz, Peter Welinder, Wojciech Zaremba, and Pieter Abbeel.
\newblock Asymmetric actor critic for image-based robot learning.
\newblock \emph{RSS}, 2018.

\bibitem[Pomerleau(1991)]{pomerleau1991efficient}
Dean~A Pomerleau.
\newblock Efficient training of artificial neural networks for autonomous navigation.
\newblock \emph{Neural computation}, 3\penalty0 (1):\penalty0 88--97, 1991.

\bibitem[Puterman(2014)]{puterman2014markov}
Martin~L Puterman.
\newblock \emph{Markov decision processes: discrete stochastic dynamic programming}.
\newblock John Wiley \& Sons, 2014.

\bibitem[Qi et~al.(2023)Qi, Kumar, Calandra, Ma, and Malik]{qi2023hand}
Haozhi Qi, Ashish Kumar, Roberto Calandra, Yi~Ma, and Jitendra Malik.
\newblock In-hand object rotation via rapid motor adaptation.
\newblock In \emph{Conference on Robot Learning}, pp.\  1722--1732. PMLR, 2023.

\bibitem[Ross et~al.(2011)Ross, Gordon, and Bagnell]{ross2011reduction}
St{\'e}phane Ross, Geoffrey Gordon, and Drew Bagnell.
\newblock A reduction of imitation learning and structured prediction to no-regret online learning.
\newblock In \emph{Proceedings of the fourteenth international conference on artificial intelligence and statistics}, pp.\  627--635. JMLR Workshop and Conference Proceedings, 2011.

\bibitem[Salter et~al.(2021)Salter, Rao, Wulfmeier, Hadsell, and Posner]{salter2021attention}
Sasha Salter, Dushyant Rao, Markus Wulfmeier, Raia Hadsell, and Ingmar Posner.
\newblock Attention-privileged reinforcement learning.
\newblock In \emph{Conference on Robot Learning}, pp.\  394--408. PMLR, 2021.

\bibitem[Schulman et~al.(2015{\natexlab{a}})Schulman, Levine, Moritz, Jordan, and Abbeel]{DBLP:journals/corr/SchulmanLMJA15}
John Schulman, Sergey Levine, Philipp Moritz, Michael~I. Jordan, and Pieter Abbeel.
\newblock Trust region policy optimization.
\newblock \emph{CoRR}, abs/1502.05477, 2015{\natexlab{a}}.
\newblock URL \url{http://arxiv.org/abs/1502.05477}.

\bibitem[Schulman et~al.(2015{\natexlab{b}})Schulman, Moritz, Levine, Jordan, and Abbeel]{Schulman2015HighDimensionalCC}
John Schulman, Philipp Moritz, Sergey Levine, Michael~I. Jordan, and P.~Abbeel.
\newblock High-dimensional continuous control using generalized advantage estimation.
\newblock \emph{CoRR}, abs/1506.02438, 2015{\natexlab{b}}.
\newblock URL \url{https://api.semanticscholar.org/CorpusID:3075448}.

\bibitem[Schulman et~al.(2017)Schulman, Wolski, Dhariwal, Radford, and Klimov]{DBLP:journals/corr/SchulmanWDRK17}
John Schulman, Filip Wolski, Prafulla Dhariwal, Alec Radford, and Oleg Klimov.
\newblock Proximal policy optimization algorithms.
\newblock \emph{CoRR}, abs/1707.06347, 2017.
\newblock URL \url{http://arxiv.org/abs/1707.06347}.

\bibitem[Seo et~al.(2023)Seo, Kim, James, Lee, Shin, and Abbeel]{seo2023multi}
Younggyo Seo, Junsu Kim, Stephen James, Kimin Lee, Jinwoo Shin, and Pieter Abbeel.
\newblock Multi-view masked world models for visual robotic manipulation.
\newblock In \emph{International Conference on Machine Learning}, pp.\  30613--30632. PMLR, 2023.

\bibitem[Sermanet et~al.(2018)Sermanet, Lynch, Chebotar, Hsu, Jang, Schaal, Levine, and Brain]{sermanet2018time}
Pierre Sermanet, Corey Lynch, Yevgen Chebotar, Jasmine Hsu, Eric Jang, Stefan Schaal, Sergey Levine, and Google Brain.
\newblock Time-contrastive networks: Self-supervised learning from video.
\newblock In \emph{2018 IEEE international conference on robotics and automation (ICRA)}, pp.\  1134--1141. IEEE, 2018.

\bibitem[Shenfeld et~al.(2023{\natexlab{a}})Shenfeld, Hong, Tamar, and Agrawal]{pmlr-v202-shenfeld23a}
Idan Shenfeld, Zhang-Wei Hong, Aviv Tamar, and Pulkit Agrawal.
\newblock {TGRL}: An algorithm for teacher guided reinforcement learning.
\newblock In Andreas Krause, Emma Brunskill, Kyunghyun Cho, Barbara Engelhardt, Sivan Sabato, and Jonathan Scarlett (eds.), \emph{Proceedings of the 40th International Conference on Machine Learning}, volume 202 of \emph{Proceedings of Machine Learning Research}, pp.\  31077--31093. PMLR, 23--29 Jul 2023{\natexlab{a}}.
\newblock URL \url{https://proceedings.mlr.press/v202/shenfeld23a.html}.

\bibitem[Shenfeld et~al.(2023{\natexlab{b}})Shenfeld, Hong, Tamar, and Agrawal]{shenfeld2023tgrl}
Idan Shenfeld, Zhang-Wei Hong, Aviv Tamar, and Pulkit Agrawal.
\newblock Tgrl: An algorithm for teacher guided reinforcement learning.
\newblock In \emph{International Conference on Machine Learning}, pp.\  31077--31093. PMLR, 2023{\natexlab{b}}.

\bibitem[Song et~al.(2018)Song, Lanka, Zhao, Yue, and Ono]{DBLP:journals/corr/abs-1804-00846}
Jialin Song, Ravi Lanka, Albert Zhao, Yisong Yue, and Masahiro Ono.
\newblock Learning to search via self-imitation.
\newblock \emph{CoRR}, abs/1804.00846, 2018.
\newblock URL \url{http://arxiv.org/abs/1804.00846}.

\bibitem[Song et~al.(2020)Song, Lanka, Yue, and Ono]{pmlr-v115-song20b}
Jialin Song, Ravi Lanka, Yisong Yue, and Masahiro Ono.
\newblock Co-training for policy learning.
\newblock In Ryan~P. Adams and Vibhav Gogate (eds.), \emph{Proceedings of The 35th Uncertainty in Artificial Intelligence Conference}, volume 115 of \emph{Proceedings of Machine Learning Research}, pp.\  1191--1201. PMLR, 22--25 Jul 2020.
\newblock URL \url{https://proceedings.mlr.press/v115/song20b.html}.

\bibitem[Sutton \& Barto(2018)Sutton and Barto]{Sutton1998}
Richard~S. Sutton and Andrew~G. Barto.
\newblock \emph{Reinforcement Learning: An Introduction}.
\newblock The MIT Press, second edition, 2018.
\newblock URL \url{http://incompleteideas.net/book/the-book-2nd.html}.

\bibitem[Tangkaratt et~al.(2021)Tangkaratt, Charoenphakdee, and Sugiyama]{pmlr-v130-tangkaratt21a}
Voot Tangkaratt, Nontawat Charoenphakdee, and Masashi Sugiyama.
\newblock Robust imitation learning from noisy demonstrations.
\newblock In Arindam Banerjee and Kenji Fukumizu (eds.), \emph{Proceedings of The 24th International Conference on Artificial Intelligence and Statistics}, volume 130 of \emph{Proceedings of Machine Learning Research}, pp.\  298--306. PMLR, 13--15 Apr 2021.
\newblock URL \url{https://proceedings.mlr.press/v130/tangkaratt21a.html}.

\bibitem[Tomar et~al.(2020)Tomar, Shani, Efroni, and Ghavamzadeh]{DBLP:journals/corr/abs-2005-09814}
Manan Tomar, Lior Shani, Yonathan Efroni, and Mohammad Ghavamzadeh.
\newblock Mirror descent policy optimization.
\newblock \emph{CoRR}, abs/2005.09814, 2020.
\newblock URL \url{https://arxiv.org/abs/2005.09814}.

\bibitem[Torabi et~al.(2018)Torabi, Warnell, and Stone]{torabi2018behavioral}
Faraz Torabi, Garrett Warnell, and Peter Stone.
\newblock Behavioral cloning from observation.
\newblock \emph{arXiv preprint arXiv:1805.01954}, 2018.

\bibitem[Vapnik \& Vashist(2009)Vapnik and Vashist]{vapnik2009new}
Vladimir Vapnik and Akshay Vashist.
\newblock A new learning paradigm: Learning using privileged information.
\newblock \emph{Neural networks}, 22\penalty0 (5-6):\penalty0 544--557, 2009.

\bibitem[Walsman et~al.(2023)Walsman, Zhang, Choudhury, Farhadi, and Fox]{Walsman2023ImpossiblyGE}
Aaron Walsman, Muru Zhang, Sanjiban Choudhury, Ali Farhadi, and Dieter Fox.
\newblock Impossibly good experts and how to follow them.
\newblock In \emph{International Conference on Learning Representations}, 2023.
\newblock URL \url{https://api.semanticscholar.org/CorpusID:259267611}.

\bibitem[Warrington et~al.(2020)Warrington, Lavington, Scibior, Schmidt, and Wood]{Warrington2020RobustAL}
Andrew Warrington, Jonathan~Wilder Lavington, A.~Scibior, Mark~W. Schmidt, and Frank~D. Wood.
\newblock Robust asymmetric learning in pomdps.
\newblock In \emph{International Conference on Machine Learning}, 2020.
\newblock URL \url{https://api.semanticscholar.org/CorpusID:229923742}.

\bibitem[Weihs et~al.(2024)Weihs, Jain, Liu, Salvador, Lazebnik, Kembhavi, and Schwing]{10.5555/3540261.3541724}
Luca Weihs, Unnat Jain, Iou-Jen Liu, Jordi Salvador, Svetlana Lazebnik, Aniruddha Kembhavi, and Alexander Schwing.
\newblock Bridging the imitation gap by adaptive insubordination.
\newblock In \emph{Proceedings of the 35th International Conference on Neural Information Processing Systems}, NIPS '21, Red Hook, NY, USA, 2024. Curran Associates Inc.
\newblock ISBN 9781713845393.

\bibitem[Wu et~al.(2024)Wu, Gu, Zhao, and Wu]{wu2024learnteachimprovesample}
Feiyang Wu, Zhaoyuan Gu, Ye~Zhao, and Anqi Wu.
\newblock Learn to teach: Improve sample efficiency in teacher-student learning for sim-to-real transfer, 2024.
\newblock URL \url{https://arxiv.org/abs/2402.06783}.

\bibitem[Xiao(2022)]{JMLR:v23:22-0056}
Lin Xiao.
\newblock On the convergence rates of policy gradient methods.
\newblock \emph{Journal of Machine Learning Research}, 23\penalty0 (282):\penalty0 1--36, 2022.
\newblock URL \url{http://jmlr.org/papers/v23/22-0056.html}.

\bibitem[Yang et~al.(2024)Yang, Zhou, Li, Tao, Li, Shen, He, Jiang, and Shi]{yang2024embodiedmultimodalagenttrained}
Yijun Yang, Tianyi Zhou, Kanxue Li, Dapeng Tao, Lusong Li, Li~Shen, Xiaodong He, Jing Jiang, and Yuhui Shi.
\newblock Embodied multi-modal agent trained by an llm from a parallel textworld, 2024.
\newblock URL \url{https://arxiv.org/abs/2311.16714}.

\bibitem[Yu et~al.(2022)Yu, Velu, Vinitsky, Gao, Wang, Bayen, and Wu]{yu2022the}
Chao Yu, Akash Velu, Eugene Vinitsky, Jiaxuan Gao, Yu~Wang, Alexandre Bayen, and Yi~Wu.
\newblock The surprising effectiveness of {PPO} in cooperative multi-agent games.
\newblock In \emph{Thirty-sixth Conference on Neural Information Processing Systems Datasets and Benchmarks Track}, 2022.

\end{thebibliography}
\bibliographystyle{iclr2026_conference}

\appendix

\section*{The Use of Large Language Models (LLMs)}
LLMs are used to polish the paper writing.

\section{Related Works}\label{app:rw}
Although leveraging historical information has proven effective for solving POMDPs \citep{igl2018deep,meng2021memory,liu2022partially}, additional information—often available in simulators during training—can be exploited to further aid learning. 
Leveraging additional information to accelerate learning in POMDPs has been explored across various frameworks and application domains \citep{vapnik2009new, lambert2018deep,lee2023learning}.
A prominent line of research focuses on Imitation Learning (IL), where expert knowledge, often equipped with extra information, significantly enhances performance in practical domains like autonomous driving \citep{bansal2018chauffeurnet,de2019causal} and robot navigation and planning \citep{choudhury2017adaptive,bhardwaj2017learning}.
However, traditional IL methods such as Behavioral Cloning (BC) \citep{pomerleau1991efficient,torabi2018behavioral} and DAgger \citep{ross2011reduction} often lead to sub-optimal solutions in scenarios requiring active information gathering by the agent \citep{pinto2018asymmetric,Warrington2020RobustAL}. 

To overcome these limitations, recent research has focused on hybrid approaches that integrate RL with IL, often in the context of policy distillation \citep{pmlr-v89-czarnecki19a}.
For instance, 
\citep{nguyen2022leveraging} modifies Soft Actor Critic (SAC) \citep{haarnoja2018soft} by replacing the entropy term with a divergence measure between agent and expert policies at each visited state.
Similarly, \citep{10.5555/3540261.3541724} introduces a balancing mechanism between BC and RL training, adjusting based on the agent’s ability to mimic the expert.
Additionally, \citep{Walsman2023ImpossiblyGE} applies potential-based reward shaping \citep{ng1999policy} 
using the expert's value function to guide the agent's policy gradient, while \citep{shenfeld2023tgrl} augments entropy in SAC to blend task reward with expert guidance, where the balance is based on the agent's performance relative to a reward-only learner.

Despite these advances, expert-driven approaches often assume access to a reliable expert, which may not be feasible when only supplementary information is available.
This has led to a growing body of work on co-training approaches where the expert and agent are learned jointly, with the expert conditioned on additional information.
For example, \citep{salter2021attention} proposes training separate policies for the agent and expert using spatial attention for image-based RL, aligning attention mechanisms through shared experiences.
\citep{pmlr-v115-song20b} co-trains two policies, each conditioned on different information, and selects the most successful rollouts from both policies to guide subsequent learning via RL or IL.
\citep{Warrington2020RobustAL} further develops this idea in adaptive asymmetric DAgger (A2D), where the expert is continuously refined through RL while supervising the agent.
\citep{wu2024learnteachimprovesample} also co-trains a teacher and a student, using the experience collected by teacher to apply RL and BC for the student.
Beyond expert-based methods, a complementary approach involves embedding supplementary information directly into the value function within the actor-critic framework \citep{pinto2018asymmetric,doi:10.1177/0278364919887447,baisero2021unbiased}, which is also called asymmetric learning.
This approach is particularly useful in multi-agent settings where global information is naturally accessible \citep{foerster2018counterfactual,lowe2017multi,yu2022the}.
Additional strategies include learning from noisy demonstrations \citep{pmlr-v130-tangkaratt21a}, improving via self-correction from past trajectories \citep{DBLP:journals/corr/abs-1804-00846}, and surpassing imperfect experts through regret-minimization frameworks \citep{pmlr-v37-changb15}. Recent work also explores leveraging LLMs as privileged experts, such as in embodied agents trained with reflective text-based guidance \citep{yang2024embodiedmultimodalagenttrained}.

Besides, there are also representation learning techniques provided in order to reconstruct the privileged information (or its latent representation) via partial observation.
For example \citet{sermanet2018time,seo2023multi} use multi-view setups (e.g., image-based manipulation with additional camera views) to learn more informative embeddings. Others \citep{lee2020learning,salter2021attention,kumar2021rma,qi2023hand} leverage privileged simulator states during training and design policies that operate on both observed and inferred states. 

In our experiments, we benchmark against several algorithms inspired by these lines of work, with detailed descriptions of the baselines provided in Appendix \ref{app:base}.

\section{Omitted Proofs}\label{app:proof}
\setcounter{proposition}{0}
\begin{proposition}
If the guider's policy is updated using policy mirror descent in each GPO iteration:
\begin{equation}
    \hat{\mu}=\text{arg}\min\{-\eta_k\langle\nabla V(\mu^{(k)}),\mu\rangle+\frac{1}{1-\gamma}\text{D}_{\mu^{(k)}}(\mu,\mu^{(k)})\},
\end{equation}
then the learner’s policy update follows a constrained policy mirror descent:
\begin{equation}
    \pi^{(k+1)}=\underset{\pi\in\Pi}{\text{arg}\min}\{-\eta_k\langle\nabla V(\pi^{(k)}),\pi\rangle+\frac{1}{1-\gamma}\text{D}_{\pi^{(k)}}(\pi,\pi^{(k)})\}
\end{equation}
\end{proposition}
\begin{proof}
First, since $D$ is a weighted sum of KL divergence, it satisfies the definition of a Bregman divergence. Therefore, for any distributions $p,q\in\Delta(A)^{|S|}$, we have
\begin{equation}
    \text{D}_q(p,q)=h_q(p)-h_q(q)-\langle\nabla h_q(q),p-q\rangle,
\end{equation}
where $h_q(p)=\sum_{s\sim d_q}p_s\log p_s$ is the negative entropy weighted by the state distribution.

Next, by backtracking $\mu^{(k)}$ to $\pi^{(k)}$ from the last time step, we get:
    \begin{equation}
    \begin{aligned}
        \hat{\mu}&=\text{arg}\min\bigg\{-\eta_k\langle\nabla V(\mu^{(k)}),\mu\rangle+\frac{1}{1-\gamma}\text{D}_{\mu^{(k)}}(\mu,\mu^{(k)})\bigg\}\\
        &=\text{arg}\min\bigg\{-\eta_k\langle\nabla V(\pi^{(k)}),\mu\rangle+\frac{1}{1-\gamma}\text{D}_{\pi^{(k)}}(\mu,\pi^{(k)})\bigg\}\\
        &=\text{arg}\min\bigg\{-(1-\gamma)\eta_k\langle\nabla V(\pi^{(k)}),\pi\rangle + h_{\pi^{(k)}}(\pi) - \langle\nabla h_{\pi^{(k)}}(\pi^{(k)}),\pi\rangle\bigg\},
    \end{aligned}
    \end{equation}
The optimality condition for $\hat{\mu}$ requires:
    \begin{equation}
        -(1-\gamma)\eta_k\nabla V(\mu^{(k)}) + \nabla h_{\mu^{(k)}}(\hat{\mu})-\nabla h_{\mu^{(k)}}(\mu^{(k)})=0,
    \end{equation}
where we use the fact that: 
\begin{equation}
        \nabla_p \text{D}_q(p,q) = \nabla_p h_q(p)-\nabla_p h_q(q).
    \end{equation}
Now, consider the update of the learner's policy, which involves a Bregman projection $\mathcal{P}_{\Pi}$:
\begin{equation}
\begin{aligned}
 \pi^{(k+1)}&=\mathcal{P}_{\Pi}(\hat{\mu})=\underset{\pi\in\Pi}{\text{arg}\min}\text{D}_{\mu^{(k)}}(\pi,\hat{\mu})\\
 &=\underset{\pi\in\Pi}{\text{arg}\min}\big\{ h_{\mu^{(k)}}(\pi)-\langle\nabla h_{\mu^{(k)}}(\hat{\mu}),\pi\rangle\big\}\\
 &=\underset{\pi\in\Pi}{\text{arg}\min}\big\{ h_{\pi^{(k)}}(\pi)-\langle \nabla h_{\pi^{(k)}}(\pi^{(k)})+(1-\gamma)\eta_k\nabla V(\pi^{(k)}),\pi\rangle\big\}\\
 &=\underset{\pi\in\Pi}{\text{arg}\min}\big\{-(1-\gamma)\eta_k\langle\nabla V(\pi^{(k)}),\pi\rangle + h_{\pi^{(k)}}(\pi) -\langle\nabla h_{\pi^{(k)}}(\pi^{(k)}),\pi\rangle\big\}\\
 &=\underset{\pi\in\Pi}{\text{arg}\min}\{-\eta_k\langle\nabla V(\pi^{(k)}),\pi\rangle+\frac{1}{1-\gamma}\text{D}_{\pi^{(k)}}(\pi,\pi^{(k)})\}
\end{aligned}
\end{equation}
This completes the proof.
\end{proof}

\begin{proposition}
    For policy $\pi$, $\mu$, $\beta$ and all state $s$, suppose $\text{D}_\text{TV}(\mu(\cdot|s),\beta(\cdot|s))\lesssim\epsilon/2$, then we have
    \begin{equation}
        \mathbb{E}_{a\sim\beta}\big[|1-\rho^\pi(s,a)|\big]\lesssim\epsilon+\sqrt{2d_{targ}}.
    \end{equation}
\end{proposition}
\begin{proof}
First, let's examine the assumption $\text{D}_\text{TV}(\mu(\cdot|s), \beta(\cdot|s)) \lesssim \epsilon/2$ to check its validity.

Notice that at the start of each PPO policy update, the importance sampling ratio $\rho^\mu(s, a)$ equals 1 because the behavioral policy is equal to the policy being updated, i.e., $\beta(a|s) = \mu(a|s)$.

As PPO proceeds, $\rho^\mu(s, a)$ is updated multiple times using the same batch of samples. Due to the clipping function applied to $\rho^\mu(s, a)$, i.e., $\text{clip}(\rho^\mu(s, a), 1-\epsilon, 1+\epsilon)$, only state-action pairs for which $\rho^\mu(s, a) \in (1-\epsilon, 1+\epsilon)$ get updated. Hence, in the early epochs of PPO, with a properly tuned step size, we expect:
\begin{equation}
    |1-\rho^\mu(s,a)|\lesssim\epsilon.
\end{equation}
Now, recalling the definition of total variation (TV) distance: 
\begin{equation}
    \text{D}_\text{TV}(\mu(\cdot|s),\beta(\cdot|s))=\frac{1}{2}\sum_a |\mu(a|s)-\beta(a|s)|=\frac{1}{2}\sum_a\beta(a|s) |\rho^\mu(s,a)-1|\lesssim\epsilon/2.
\end{equation}
This confirms that the assumption $\text{D}_\text{TV}(\mu(\cdot|s), \beta(\cdot|s)) \lesssim \epsilon/2$ is reasonable, especially for the first few policy updates.

By the triangle inequality for total variation distance: 
\begin{equation}
    \text{D}_\text{TV}(\pi(\cdot|o),\beta(\cdot|s))\le D_{TV}(\pi(\cdot|o),\mu(\cdot|s))+\text{D}_\text{TV}(\mu(\cdot|s),\beta(\cdot|s)),
\end{equation}
we have 
\begin{align*}
    \text{D}_\text{TV}(\pi(\cdot|o),\beta(\cdot|s))&\le 
    \sqrt{\frac{1}{2}\text{D}_{\text{KL}}(\pi(\cdot|o),\mu(\cdot|s))} + \text{D}_\text{TV}(\mu(\cdot|s),\beta(\cdot|s))\\
    &\lesssim \sqrt{\frac{1}{2}d_{\text{targ}}} + \epsilon/2,
\end{align*}
where we use Pinsker's inequality to bound the total variation distance between $\pi$ and $\mu$ in terms of their KL divergence.

Finally, since total variation is linked to the expected difference between probabilities under different policies, we have:
\begin{equation}
    \mathbb{E}_{a\sim\beta}\big[|1-\rho^\pi(s,a)|\big]=2\text{D}_\text{TV}(\pi(\cdot|o),\beta(\cdot|s))\lesssim\epsilon+\sqrt{2d_{targ}}.
\end{equation}
This result implies that, under the assumption, the majority of samples are valid for updating the learner's policy during the early PPO epochs.
\end{proof}

\section{Pseudo Code}\label{app:alg}

In this section, we present the pseudo code of our algorithm (see Algorithm \ref{alg}). The algorithm is based on PPO, with an additional objective to leverage the extra information available during training.

\begin{algorithm}[ht]
\caption{Guided Policy Optimization}
\label{alg}
\KwIn{Initial policy parameters $\theta_0$, value function parameters $\phi_0$}

\For{$k = 0, 1, 2, \dots$}{
  Collect trajectory set $\mathcal{D}_K = \{\tau_i\}$ by running guider policy $\mu_k = \mu(\cdot|o_g; \theta_k)$\;
  Compute rewards-to-go $\hat{R}_t$\;
  Compute advantage estimates $\hat{A}_t$ using GAE w.r.t.\ current value function $V_{\phi_k}$\;
  Update policy parameters $\theta_k$ to $\theta_{k+1}$ by maximizing GPO objective (\ref{eq:gpo-p}) or (\ref{eq:gpo-c})\;
  Fit value function parameters $\phi_{k+1}$ by minimizing mean squared error:\;
  \Indp
  \begin{equation*}
    \phi_{k+1} = \arg\min_\phi \frac{1}{|\mathcal{D}_K|T} \sum_{\tau \in \mathcal{D}_K} \sum_{t=0}^T \left(V_\phi((o_g)_t) - \hat{R}_t\right)^2
  \end{equation*}
  \Indm
}
\end{algorithm}

\section{GPO on TigerDoor-alt Problem}\label{app:tiger}
\begin{table}[ht]
        \centering
\caption{TigerDoor-alt problem}
\begin{tabular}{|c|c|c|}
\hline
\diagbox{state}{action} & $a_L$ & $a_R$  \\
\hline
$s_L$   & 2   & 0  \\
\hline
$s_R$   & 0   & 1  \\
\hline
\end{tabular}
\end{table}

Here we provide an intuitive example to illustrate how GPO can achieve the optimal policy in the TigerDoor-alt problem.
At time step $t$, suppose the guider’s and learner’s policies are:
\begin{align*}
    \mu_t(\cdot|s_L)=\mu_t(\cdot|s_R)=\pi_t=(x_t,y_t),
\end{align*}
where the two policies are equal due to the backtracking from time step $t-1$.
After one update step, the guider's policy becomes:
\begin{align*}
    \hat\mu_{t}(\cdot|s_L)=(x_t+p_t,y_t-p_t), \ \ \hat\mu_t(\cdot|s_R)=(x_t+q_t,y_t-q_t)
\end{align*}
The key insight is that the higher reward for $(s_L,a_L)$ compared to $(s_R,a_R)$ leads to a larger gradient step, implying $p_t>q_t$.
The learner then imitates the guider, resulting in the updated policy:
\begin{align*}
    \pi_{t+1}=(x_t+\frac{p_t-q_t}{2},y_t-\frac{p_t-q_t}{2}).
\end{align*}
Hence, $\pi_{t+1}(a_L)>\pi_{t}(a_L)$, meaning the learner’s policy moves closer to the optimal policy $(1,0)$ in a monotonic fashion.

The critical mechanism is that actions yielding higher rewards induce larger updates in the guider's policy, which the learner then captures via imitation. Meanwhile, the backtracking step keeps the guider aligned with the learner, enabling steady and consistent policy improvement.

\section{Experimental Settings}\label{app:exp}

\subsection{Baselines}\label{app:base}
In this section, we briefly introduce the baselines used in our experiments.

\textbf{PPO}. This is the standard algorithm used to train the agent without any additional information. The objective function is given by:
\begin{equation}
    \mathcal{L}(\pi)=-\mathbb{E}\bigg[\min
    \bigg(\rho^\pi(o_l,a)A^\beta(o_l,a),
    \rho_{clip}^\pi(o_l,a,\epsilon)A^\beta(o_l,a)\bigg)\bigg],
\end{equation}
where the behavioral policy is $\beta=\pi_{\text{old}}$.

\textbf{GPO-naive}. This variant of GPO uses the GPO-penalty without the auxiliary RL loss term. The objective function is:
\begin{equation}
\begin{aligned}
    \mathcal{L}_{\text{GPO-naive}}(\theta)=-\mathbb{E}\Big[\min
    \Big(\rho^{\mu_\theta}&A^\beta(o_g,a),
    \rho_{clip}^{\mu_\theta}A^\beta(o_g,a)\Big)-\alpha\text{D}_{\text{KL}}\big(\mu_{\theta}(\cdot|o_l),\pi_{\hat{\theta}}(\cdot|o_g)\big)\\
    &-\text{D}_{\text{KL}}\big(\mu_{\hat{\theta}}(\cdot|o_l),\pi_{\theta}(\cdot|o_g)\big)
    \Big].
\end{aligned}
\end{equation}

\textbf{GPO-ablation}. This is another variant of GPO-penalty, but without the BC loss term. The objective is:
\begin{equation}
\begin{aligned}
    \mathcal{L}_{\text{GPO-ablation}}(\theta)=-\mathbb{E}\Big[\min
    \Big(&\rho^{\mu_\theta}A^\beta(o_g,a),
    \rho_{clip}^{\mu_\theta}A^\beta(o_g,a)\Big)-\alpha\text{D}_{\text{KL}}\big(\mu_{\theta}(\cdot|o_l),\pi_{\hat{\theta}}(\cdot|o_g)\big)\\
    &+\min
    \Big(\rho^{\pi_\theta}A^\beta(o_g,a),
    \rho_{clip}^{\pi_\theta}A^\beta(o_g,a)\Big).
\end{aligned}
\end{equation}

\textbf{PPO-asym}. This method trains the student using PPO, but with asymmetric value function taking $o_g$ as input. The objective is:
\begin{equation}
    \mathcal{L}(\pi)=-\mathbb{E}\bigg[\min
    \bigg(\rho^\pi(o_l,a)A^\beta(o_g,a),
\rho_{clip}^\pi(o_l,a,\epsilon)A^\beta(o_g,a)\bigg)\bigg].
\end{equation}

\textbf{ADVISOR}. Given teacher's policy $\mu$, ADVISOR \citep{10.5555/3540261.3541724} uses a balancing coefficient $w$ between BC and RL training, based on the distance between the teacher's policy $\mu$ and an auxiliary imitation policy $\hat{\pi}$:
\begin{align*}
    \mathcal{L}(\pi)=-\mathbb{E}\bigg[w\text{CE}(\mu(\cdot|o_g),\pi(\cdot|o_l))+(1-w)\min
    \Big(\rho^\pi(o_l,a)A^\beta(o_l,a),
\rho_{clip}^\pi(o_l,a,\epsilon)A^\beta(o_l,a)\Big)\bigg],
\end{align*}
where $w=exp(-\alpha \text{D}_{\text{KL}}(\mu(\cdot|o_g),\hat{\pi}(\cdot|o_l)))$ and CE means cross-entropy.

\textbf{ADVISOR-co}. This is a modified version of the ADVISOR algorithm for co-training setting, as the original does not involve teacher training. The teacher’s objective is:
\begin{equation}
    \mathcal{L}(\mu)=-\mathbb{E}\bigg[\min
    \bigg(\rho^\mu(o_g,a)A^\beta(o_g,a),
    \rho_{clip}^\mu(o_g,a,\epsilon)A^\beta(o_g,a)\bigg)\bigg].
\end{equation}
ADVISOR-co can be viewed as GPO-penalty without the backtrack term and with a different $\alpha$-update schedule. However, in the absence of backtracking, the coefficient $w$ quickly diminishes, as the auxiliary policy cannot follow the teacher effectively, reducing this approach to pure PPO training for the student.

\textbf{PPO+BC}. In this approach, the teacher is trained using PPO:
\begin{equation}
    \mathcal{L}(\mu)=-\mathbb{E}\bigg[\min
    \bigg(\rho^\mu(o_g,a)A^\beta(o_g,a),
    \rho_{clip}^\mu(o_g,a,\epsilon)A^\beta(o_g,a)\bigg)\bigg],
\end{equation}
while the student is trained using BC with the teacher:
\begin{equation}
L(\pi)=\mathbb{E}\big[\text{D}_{\text{KL}}\big(\mu(\cdot|o_g),\pi(\cdot|o_l)\big)\big].
\end{equation}

\textbf{PPO+BC-t}. Given teacher's policy $\mu$, the student is trained using a combined loss of PPO and BC:
\begin{equation}
\mathcal{L}(\pi)=-\mathbb{E}\big[\min
    \bigg(\rho^\mu(o_g,a)A^\beta(o_g,a),
    \rho_{clip}^\mu(o_g,a,\epsilon)A^\beta(o_g,a)\bigg)-\text{D}_{\text{KL}}\big(\mu(\cdot|o_g),\pi(\cdot|o_l)\big)\big].
\end{equation}

\textbf{A2D}. Adaptive Asymmetric DAgger (A2D) \citep{Warrington2020RobustAL} is closely related to GPO, as it also involves co-training both the teacher and the student.  
A2D uses a mixture policy $\beta(a|o_g,o_l)=\lambda\mu(a|o_g) + (1-\lambda)\pi(a|o_l)$ to collect trajectories and train the expert $\mu$ with a mixed value function $V(o_g,o_l)=\lambda V^\mu(o_g) + (1-\lambda)v^\pi(o_l)$. The objective is: 
\begin{equation}
    \mathcal{L}(\mu)=-\mathbb{E}\bigg[\min
    \bigg(\rho^\mu(o_g,o_l,a)A^\beta(o_g,o_l,a),
    \rho_{clip}^\mu(o_g,o_l,a,\epsilon)A^\beta(o_g,o_l,a)\bigg)\bigg],
\end{equation}
while the student is updated through BC:
\begin{equation}
\mathcal{L}(\pi)=\mathbb{E}\big[\text{D}_{\text{KL}}\big(\mu(\cdot|o_g),\pi(\cdot|o_l)\big)\big]
\end{equation}
In practice, A2D sets $\lambda = 0$ or anneals it quickly for better performance. When $\lambda = 0$, A2D is equivalent to GPO-naive without the backtrack step, and it uses the student's behavioral policy $\pi$ instead of the teacher's policy $\mu$.
While A2D implicitly constrains the teacher’s policy through the PPO clipping mechanism (which prevents the teacher from deviating too far from the student’s behavioral policy), this is insufficient to replace the explicit backtrack step.
The gap between $\mu$ and $\pi$ can accumulate if the student fails to follow the teacher.
Consequently, most samples will be clipped as training progresses, leading A2D to struggle in training a strong teacher.
  
\textbf{ELF}. Given teacher's policy $\mu$, ELF Distillation \citep{Walsman2023ImpossiblyGE} trains two policies jointly: a follower $\pi_f$ to mimic the teacher through BC:
\begin{equation}
\mathcal{L}(\pi_f)=\mathbb{E}\big[\text{D}_{\text{KL}}\big(\mu(\cdot|o_g),\pi_f(\cdot|o_l)\big)\big],
\end{equation}
and a explorer $\pi_e$ trained through PPO:
\begin{equation}
    \mathcal{L}(\pi_e)=-\mathbb{E}\bigg[\min
    \bigg(\rho^{\pi_e}(o_l,a)A(o_l,a),
    \rho_{clip}^{\pi_e}(o_l,a,\epsilon)A(o_l,a)\bigg)\bigg],
\end{equation}
To utilize teacher supervision, ELF applies a potential-based reward shaping \citep{ng1999policy} $r + \gamma V^{\pi_f}(o_l')-V^{\pi_f}(o_l)$ to the explorer, where $V^{\pi_f}$ is the value function of follower. 
However, ELF needs to divide the interaction equally to train the follower and the explorer, which leads to inefficiencies.

\textbf{ELF-asym}. Since the follower is not required during execution, an asymmetric value function $V^{\pi_f}(o_g)$ is used instead of the original one.
Although there are some performance improvement, ELF-asym still performs worse than PPO-asym due to inefficient experience usage.

\textbf{L2T-PPO}. Similar to PPO+BC, the teacher is trained using PPO:
\begin{equation}
    \mathcal{L}(\mu)=-\mathbb{E}\bigg[\min
    \bigg(\rho^\mu(o_g,a)A^\beta(o_g,a),
    \rho_{clip}^\mu(o_g,a,\epsilon)A^\beta(o_g,a)\bigg)\bigg],
\end{equation}
while the student is trained using a combined loss of PPO and BC with the teacher:
\begin{equation}
\mathcal{L}(\pi)=-\mathbb{E}\bigg[\min
    \bigg(\rho^\mu(o_g,a)A^\beta(o_g,a),
    \rho_{clip}^\mu(o_g,a,\epsilon)A^\beta(o_g,a)\bigg)-\text{D}_{\text{KL}}\big(\mu(\cdot|o_g),\pi(\cdot|o_l)\big)\bigg],
\end{equation}
where the behavioral policy $\beta=\mu$.

\subsection{Hyperparameters}\label{sec:hyper}
The experiments in Sections \ref{sec:exp1} and \ref{sec:exp3} use the same codebase from \citep{lu2023structured}.
The hyperparameters for these experiments are listed in Table \ref{app:pop}.
For GPO-clip, due to the asymmetry with large $\delta$, we replace the clip$(\frac{\mu}{\pi},1-\delta,1+\delta)$ with clip$(\frac{\mu}{\pi},\frac{1}{\rho},\rho)$ in the POPGym tasks.

For the experiments in Section \ref{sec:exp2}, 
we use the codebase from \citep{brax2021github}. 
We perform a hyperparameter search for the original versions of the tasks and then fix the same hyperparameters for the partially observable and noisy variants. 
The hyperparameter search is detailed in Table \ref{app:search}, and the selected hyperparameters for the experiments are provided in Table \ref{app:mujoco}.
Other fixed hyperparameters are listed in Table \ref{app:mujoco1}.

All algorithms in Brax are run with 10 random seeds, whereas those in POPGym use 3 seeds, as the latter exhibits lower variance. 
Reward curves in this paper report the mean and standard deviation across these runs.

\begin{table}[ht]
    \centering
\caption{Hyperparameters used in TigerDoor and POPGym.}
\begin{tabular}{c|c|c}
\hline
Parameter & Value (TigerDoor) & Value (POPGym)\\
\hline
Adam Learning Rate & 5e-5& 5e-5\\
Number of Environments & 64& 64\\
Unroll Length & 1024 & 1024 \\
Number of Timesteps  & 2e6 & 15e6 \\
Number of Epochs &  30 &  30 \\
Number of Minibatches &  8 &  8 \\
Discount $\gamma$ &  0.99 &  0.99\\
GAE $\lambda$ &  1.0 &  1.0 \\
Clipping Coefficient $\epsilon$ &  0.2 &  0.2 \\
Entropy Coefficient &  0.0 &  0.0 \\
Value Function Weight &  1.0 &  1.0 \\
Maximum Gradient Norm &  0.5 &  0.5 \\
Activation Function  &   LeakyReLU &   LeakyReLU \\
Encoder Layer Sizes  &   128 &   [128,256] \\
Recurrent Layer Hidden Size  & - & 256 \\
Action Decoder Layer Sizes  &  128 &  [128,128] \\
Value Decoder Layer Sizes &   128  &  [128,128] \\
KL Threshold $d$ & 0.001   &  0.1 (0.001 for CartPole) \\
Clip $\rho$ & 1.1 & 10 (1.2 for CartPole)\\
RL Coefficient $\alpha$ & 1 & 0 (1 for CartPole)\\
\end{tabular}

\label{app:pop}
\end{table}

\begin{table}[ht]
    \centering
\caption{Sweeping procedure in the Brax domain.}    
\begin{tabular}{c|c}
\hline
Parameter & Value\\
\hline
Reward Scaling $r_s$ & [0.1, 1]\\
Discount $\gamma$ & [0.97, 0.99, 0.997]\\
Unroll Length $l$ & [5, 10, 20]\\
Batchsize $b$ & [256, 512, 1024]\\
Number of Minibatches $n$ & [4, 8, 16, 32]\\
Number of Epochs $e$ & [2, 4, 8]\\
Entropy Coefficient $c$ & [0.01, 0.001]\\
KL Threshold $d$ & [0.01, 0.001]\\
Clip $\delta$ & [0.1, 0.3]\\
RL Coefficient $\alpha$ & [0, 2, 3]\\
\end{tabular}
\label{app:search}
\end{table}

\begin{table}[ht]
    \centering
\caption{Adopted hyperparameters in the Brax domain. Notations correspond to Table \ref{app:search}.}    
\begin{tabular}{c|cccccccccc}
\hline
Task & $r_s$ & $\gamma$ & $l$ & $b$ & $n$ & $e$ & $c$ & $d$ & $\delta$ & $\alpha$\\
\hline
Ant         & 0.1 & 0.97 & 5 & 1024 & 32 & 4 & 0.01  & 0.001 & 0.3 & 2\\
Halfcheetah & 1   & 0.99 & 5 & 512  & 4  & 4 & 0.001 & 0.001 & 0.1 & 2\\
Humanoid    & 0.1 & 0.99 & 5 & 512  & 32 & 4 & 0.01  & 0.001 & 0.1 & 2\\
HumanoidStandup&0.1& 0.99& 5 & 256  & 32 & 8 & 0.01  & 0.001 & 0.3 & 3\\
InvertedDoublePendulum&1& 0.997& 20 & 256  & 8 & 4 & 0.01  & 0.001 & 0.1 & 0\\
Swimmer     & 1   & 0.997 & 5 & 256  & 32  & 4 & 0.01 & 0.001 & 0.3 & 3\\
Walker2d    & 1   & 0.99  & 5 & 512  & 32  & 4 & 0.01 & 0.001 & 0.1 & 2\\
\end{tabular}

\label{app:mujoco}
\end{table}

\begin{table}[ht]
    \centering
\caption{Common hyperparameters used in Brax domain.}    
\begin{tabular}{c|c}
\hline
Parameter & Value\\
\hline
Adam Learning Rate  & 3e-4\\
Number of Environments & 2048\\
Episode Length  & 1024\\
Number of Timesteps   & 3e7\\
GAE $\lambda$ &  0.95 \\
Clipping Coefficient $\epsilon$ &  0.3 \\
Activation Function  &   SiLU \\
Value Layer Sizes & [128, 128]\\
Policy Layer Sizes & [128, 128]\\
\end{tabular}

\label{app:mujoco1}
\end{table}

\subsection{Environment Descriptions}
We provide a brief overview of the environments used and the guider's observation settings.

\textbf{Brax tasks and CartPole in POPGym}: 
For these tasks, velocities and angular velocities are removed from the learner’s observation. Gaussian noise with standard deviations of 0.1, 0.2, and 0.3 is added to the observations, corresponding to the difficulty levels \textit{Easy}, \textit{Medium}, and \textit{Hard}, respectively.
The guider, however, has access to the noiseless observations and the removed velocities.

\textbf{Autoencode}: During the WATCH phase, a deck of cards is shuffled and played in sequence to the agent with the watch indicator set. The watch indicator is unset at the last card in the sequence, where the agent must then output the sequence of cards in order. 
The guider directly observes the correct card to be output at each timestep.

\textbf{Battleship}: A partially observable version of Battleship game, where the agent has no access to the board and must derive its own internal representation. 
Observations contain either HIT or MISS and the position of the last salvo fired. The player receives a positive reward for
striking a ship, zero reward for hitting water, and negative reward for firing on a specific tile more than once.
The guider has access to a recorder that tracks all previous actions taken by the agent.

\textbf{Count Recall}: Each turn, the agent receives a next value and query value. The agent must answer the query with the number of occurrences of a specific value. In other words, the agent must store running counts of each unique observed value, and report a specific count back, based on the query value. 
The guider directly observes the running counts at each timestep.

\textbf{Repeat Previous}: At the first timestep, the agent receives one of four values and a remember indicator. Then it randomly receives one of the four values at each successive timestep without the remember indicator.
The agent is rewarded for outputting the observation from some constant k timesteps ago, i.e. observation $o_{t-k}$ at time $t$.
The guider has direct access to the value $o_{t-k}$ at time $t$.

\subsection{Additional Comparative Experiments}\label{app:teacher}

\begin{figure}[ht]
\centering
\begin{subfigure}{0.99\columnwidth}
  \setlength{\abovecaptionskip}{5pt}
        \centering
\includegraphics[width=0.99\linewidth]{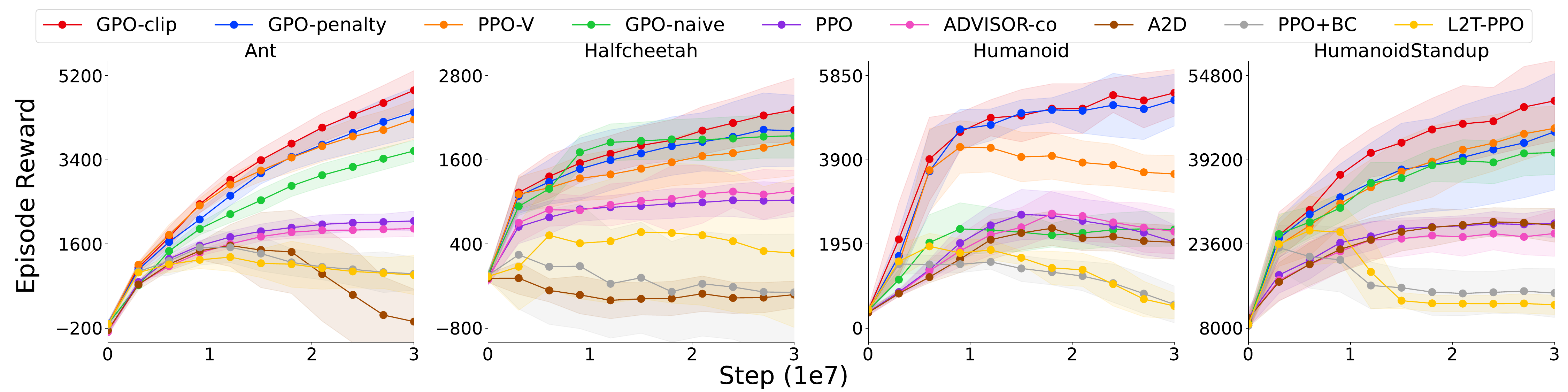}
  \end{subfigure}
  \caption{Performance comparison on selected Brax tasks, including L2T-PPO.}\label{fig:l2t}
\end{figure}

Here, we present additional baselines that were not included in the main paper.
First, L2T-RL tackles a problem similar to ours by employing a fully observable teacher to supervise a partially observable student. However, L2T-RL resembles PPO+BC, as it uses the teacher purely as a behavioral policy without aligning it with the student’s policy, which results in ineffective supervision. Moreover, L2T-RL relies on the teacher’s experience to train the student through RL without any offline adaptation, further limiting the effectiveness of the RL training. Fig. \ref{fig:l2t} illustrates the performance of the PPO-based L2T-RL (details in Appendix~\ref{app:base}), where L2T-PPO performs similarly to PPO+BC and falls short of the other methods proposed in our work.

Second, we provide a more detailed comparison with methods that train a teacher first and then apply TSL techniques to address the challenge of an inimitable teacher. Two state-of-the-art approaches in this category are ADVISOR and TGRL, which are based on PPO and SAC, respectively. We evaluate these methods using a PPO-trained teacher with full observability on the Ant task. The results, shown in Fig. \ref{fig:teacher-app}, indicate that while both methods perform well and demonstrate improved efficiency over their respective base algorithms under full observability, their performance degrades in the partially observable Ant task. In this case—where the teacher's policy is effectively inimitable—the TSL methods perform comparably to their base RL algorithms.

Fig. \ref{fig:teacher-kl} shows the KL divergence between the agent policies of these TSL methods and the teacher. Under full observability, where the teacher was trained, the agents can successfully mimic the teacher’s policy. However, under partial observability, the agents struggle to imitate the teacher’s behavior, leading to a substantial KL divergence. Since both ADVISOR and TGRL revert to standard RL when the teacher becomes inimitable, this explains their performance similarity to the base algorithms in such scenarios.

Additionally, we report TGRL’s performance across the 28 Brax tasks used in our experiments (see Fig. \ref{app:tgrl}). Note that TGRL follows an off-policy training paradigm, in contrast to all other methods presented in the main paper, which makes it significantly slower (refer to Table~\ref{app:comp}). Therefore, we run TGRL for only 1M steps, which is sufficient for convergence as shown by the learning curves.

Finally, we include the representation learning method RMA~\citep{kumar2021rma}, which aims to reconstruct the latent privileged information used by the teacher during the student training phase. Such representation learning approaches are promising when privileged information can be reliably approximated from partial observations. However, their effectiveness is limited in Brax, where observations are noisy (Figure~\ref{app:tgrl}). Since the noise cannot be removed by a deterministic mapping without additional structure or assumptions, regression-based reconstruction tends to collapse to identity mappings and fails to recover meaningful latent representations. Moreover, current encoder-based pipelines often lack theoretical guarantees for convergence or generalization across diverse tasks, particularly under partial observability.

\begin{figure}[ht]
\centering
\begin{subfigure}{0.46\columnwidth}
  \setlength{\abovecaptionskip}{5pt}
        \centering
\includegraphics[width=0.99\linewidth]{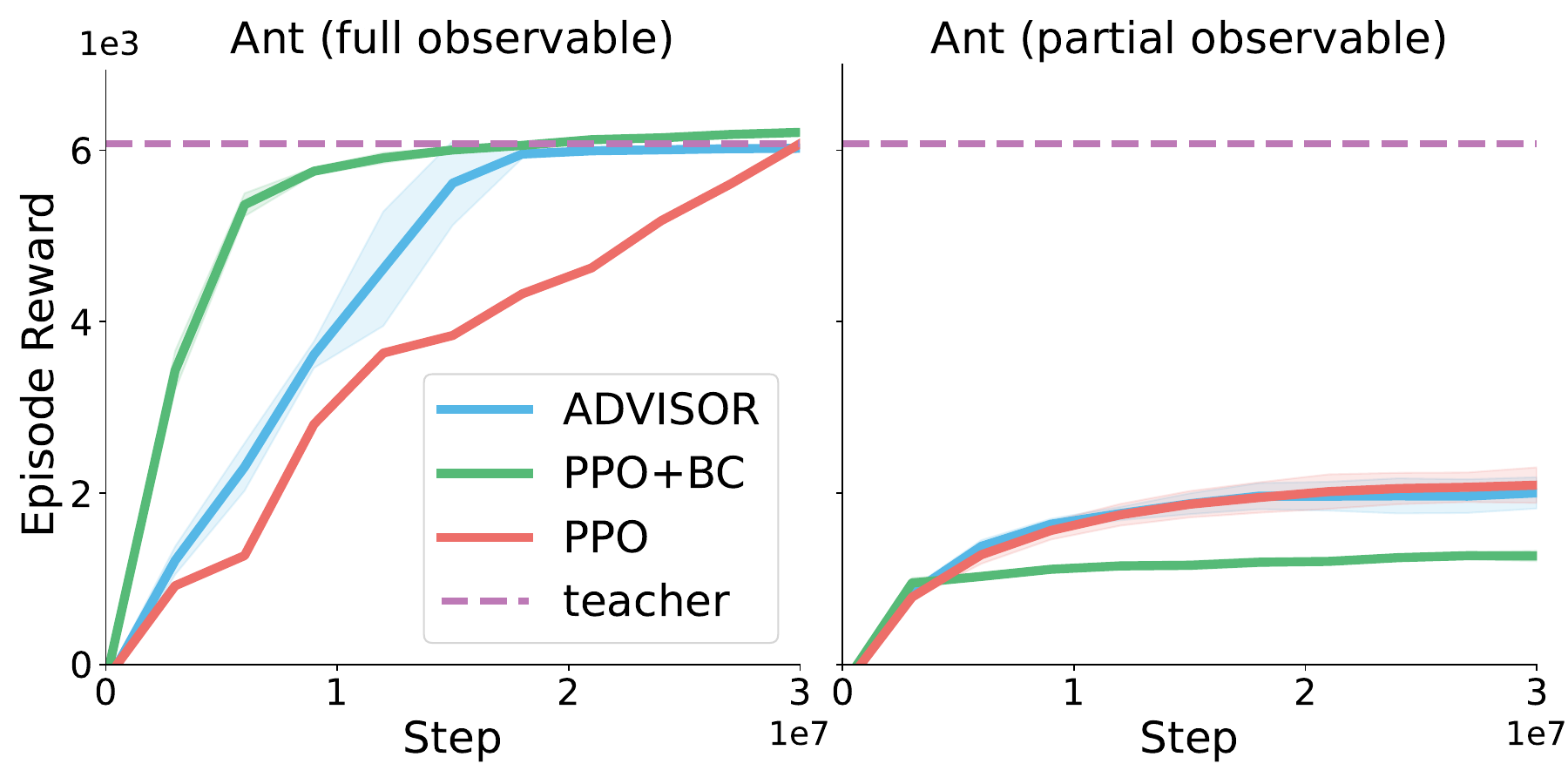}
  \end{subfigure}
\hspace{0.2cm}
\begin{subfigure}{0.46\columnwidth}
  \setlength{\abovecaptionskip}{5pt}
        \centering
\includegraphics[width=0.99\linewidth]{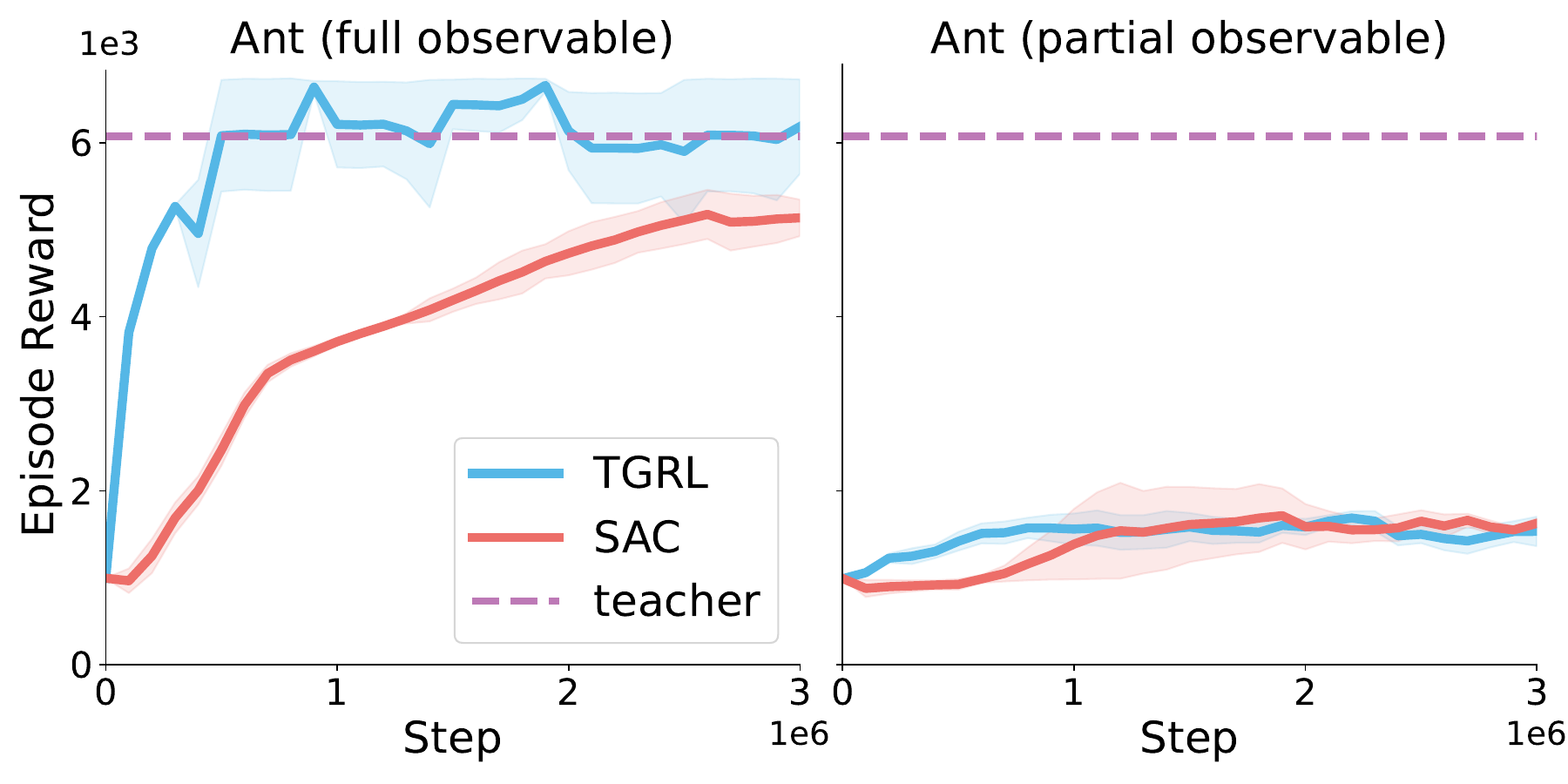}
  \end{subfigure}
  \caption{ADVISOR, PPO+BC and TGRL with a pre-trained teacher.}\label{fig:teacher-app}
\end{figure}

\begin{figure}[ht]
\centering
\begin{subfigure}{0.46\columnwidth}
  \setlength{\abovecaptionskip}{5pt}
        \centering
\includegraphics[width=0.99\linewidth]{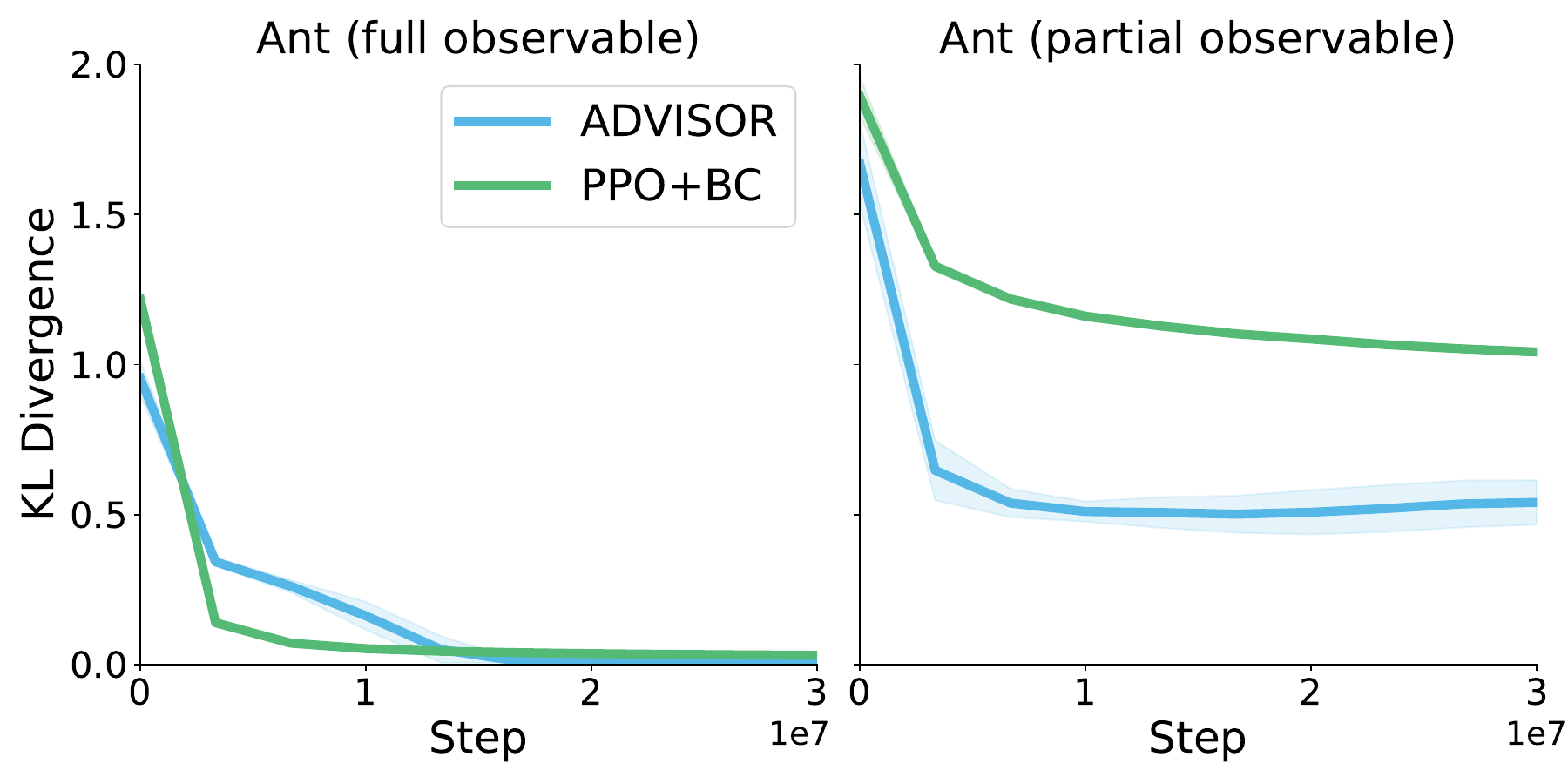}
  \end{subfigure}
\hspace{0.2cm}
\begin{subfigure}{0.46\columnwidth}
  \setlength{\abovecaptionskip}{5pt}
        \centering
\includegraphics[width=0.99\linewidth]{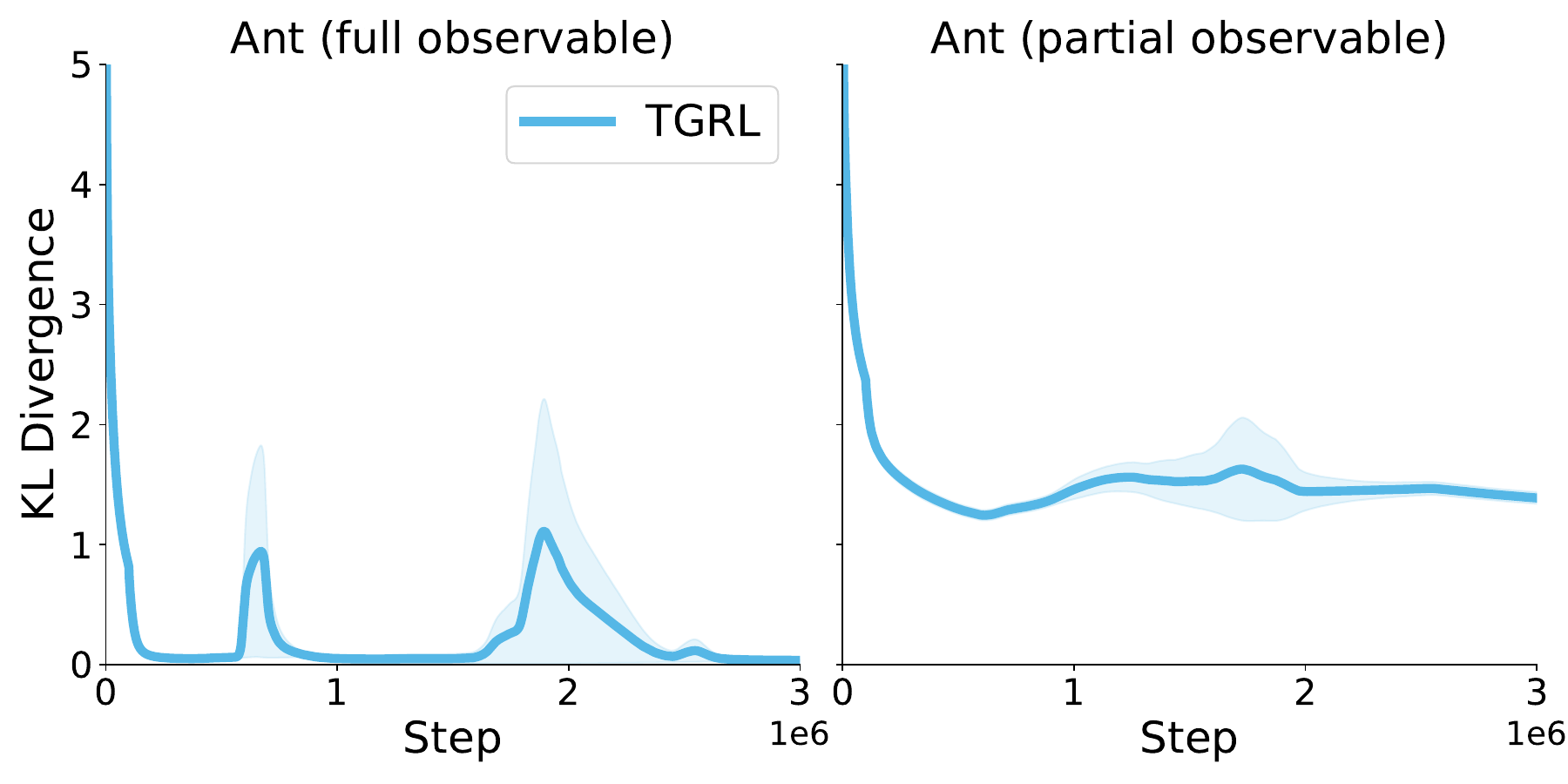}
  \end{subfigure}
  \caption{The KL divergence of ADVISOR, PPO+BC and TGRL with pre-trained teacher.}\label{fig:teacher-kl}
\end{figure}

\begin{figure}
\centering
\begin{subfigure}{\columnwidth}
  \setlength{\abovecaptionskip}{5pt}
        \centering
\includegraphics[width=1\linewidth]{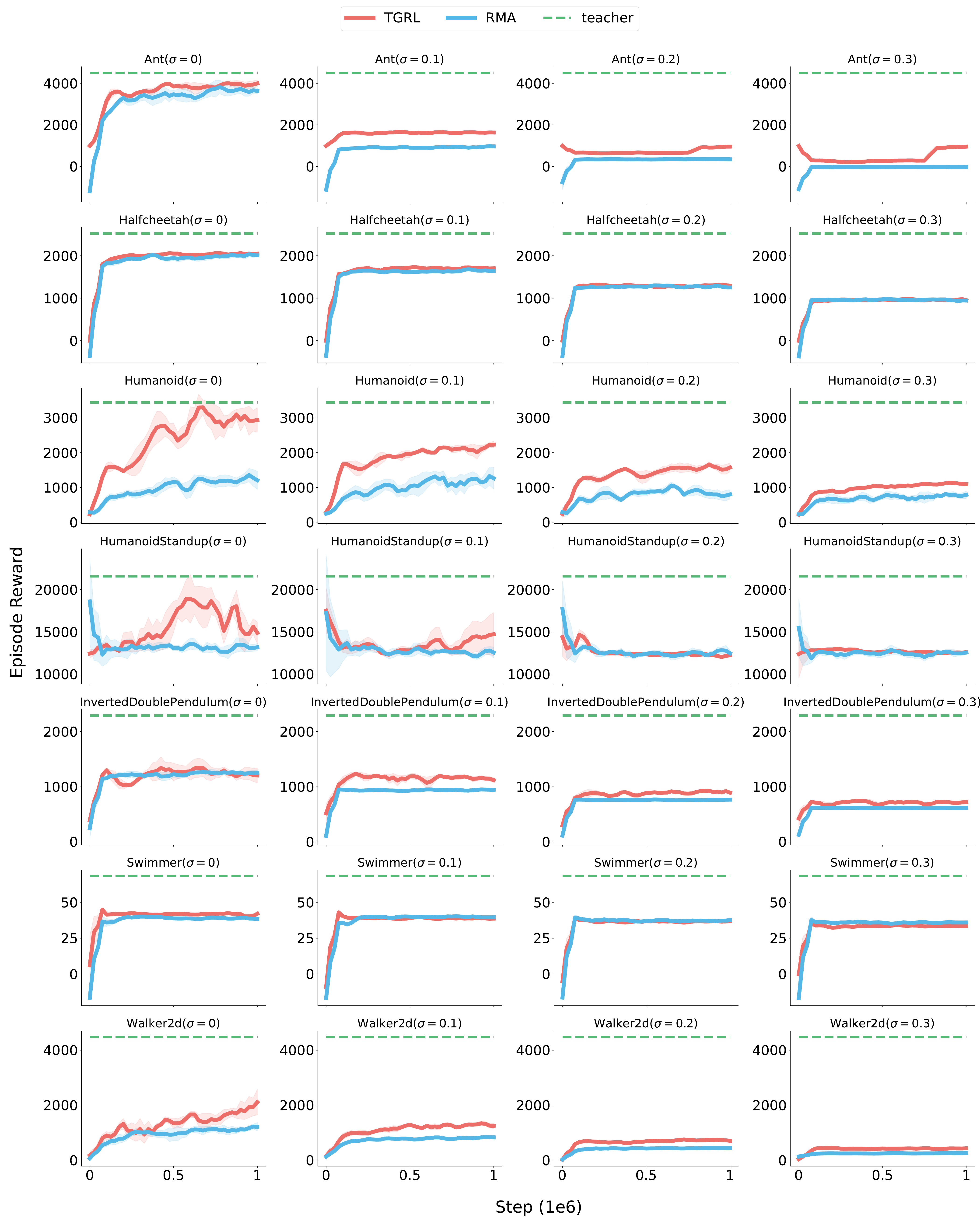}
  \end{subfigure}
  \caption{TGRL and RMA with pre-trained teacher on 28 Brax tasks.}\label{app:tgrl}
\end{figure}

\subsection{Additional Figures}
Fig. \ref{app:brax} shows the reward curves of the experiments presented in Section \ref{sec:exp2}.
Fig. \ref{app:share} illustrates the performance influenced by the parameter sharing.
We can observe that parameter sharing can sometimes impair performance, particularly when the observation dimension is large. For instance, in the \textit{HumanoidStandup} task, the observation dimension is 400, which challenges the expressive capacity of the network. Thus, the decision to share the policy network represents a trade-off between memory efficiency and performance.

\begin{figure}
\centering
\begin{subfigure}{\columnwidth}
  \setlength{\abovecaptionskip}{5pt}
        \centering
\includegraphics[width=1\linewidth]{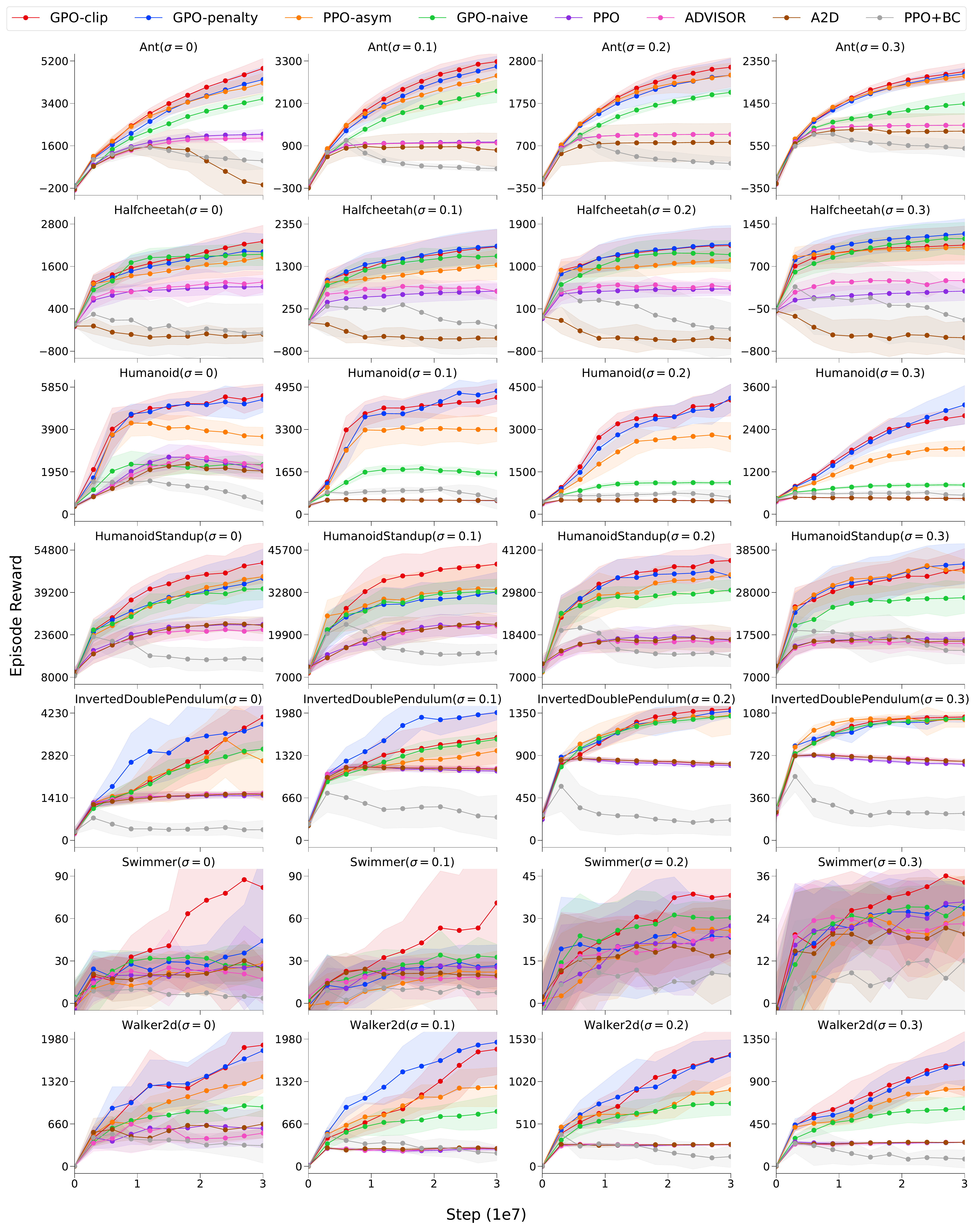}
  \end{subfigure}
  \caption{Comparing betweetn GPO and other baselines on 28 Brax tasks.}\label{app:brax}
\end{figure}

\begin{figure}[ht]
\centering
\begin{subfigure}{\columnwidth}
  \setlength{\abovecaptionskip}{5pt}
        \centering
\includegraphics[width=1\linewidth]{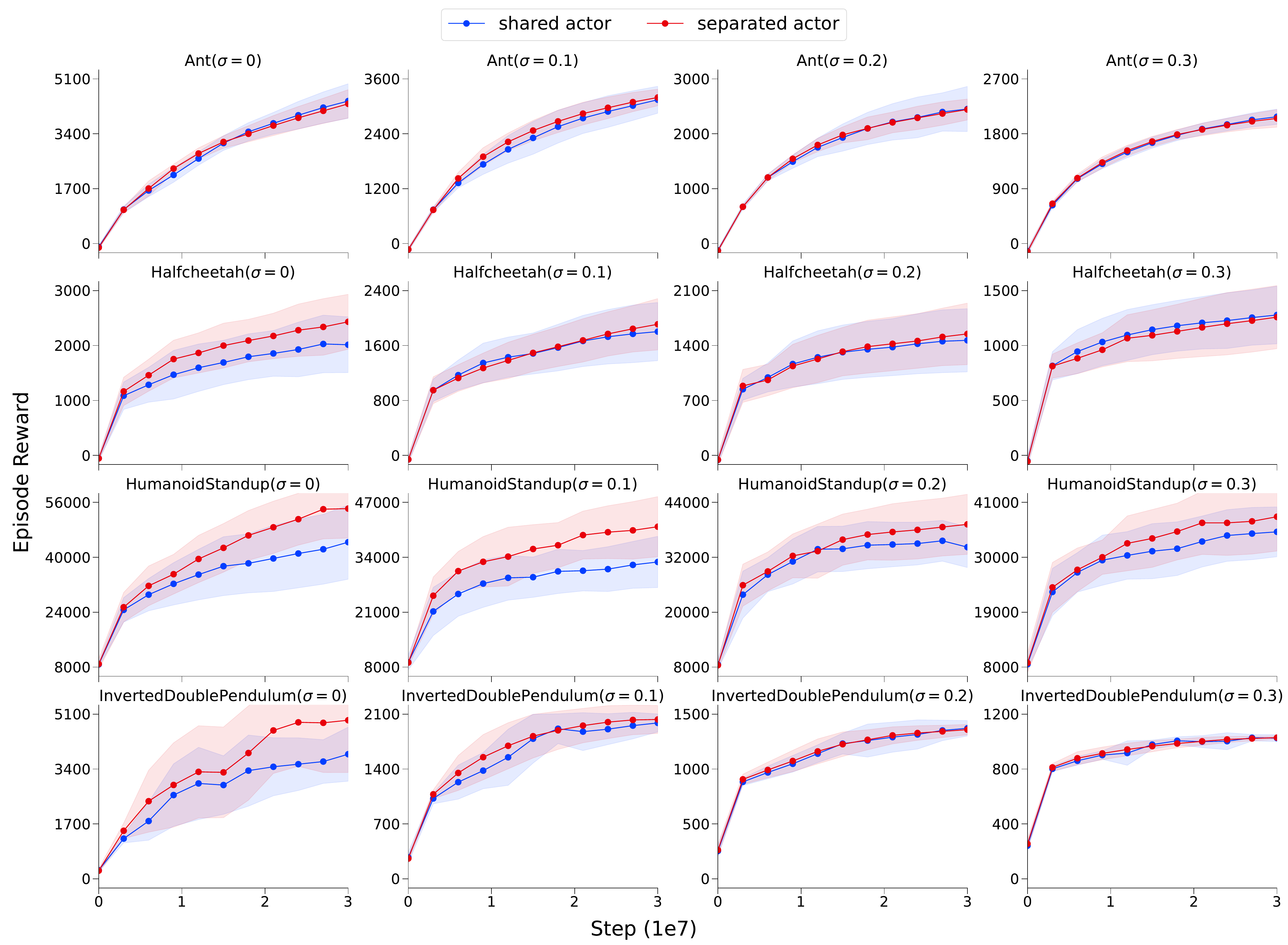}
  \end{subfigure}
  \caption{Comparing shared and separated policy networks of GPO-penalty.}\label{app:share}
\end{figure}

\begin{figure}[ht]
\centering
\begin{subfigure}{\columnwidth}
  \setlength{\abovecaptionskip}{5pt}
        \centering
\includegraphics[width=1\linewidth]{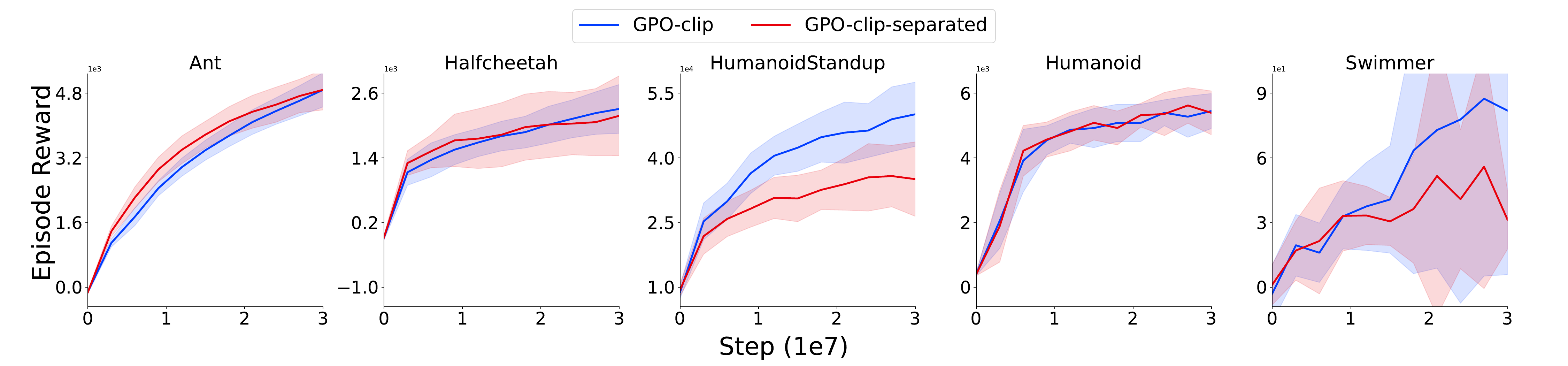}
  \end{subfigure}
  \caption{Comparing joint and separated update of GPO-clip.}
\end{figure}

\begin{figure}[ht]
\centering
\begin{subfigure}{\columnwidth}
  \setlength{\abovecaptionskip}{5pt}
        \centering
\includegraphics[width=1\linewidth]{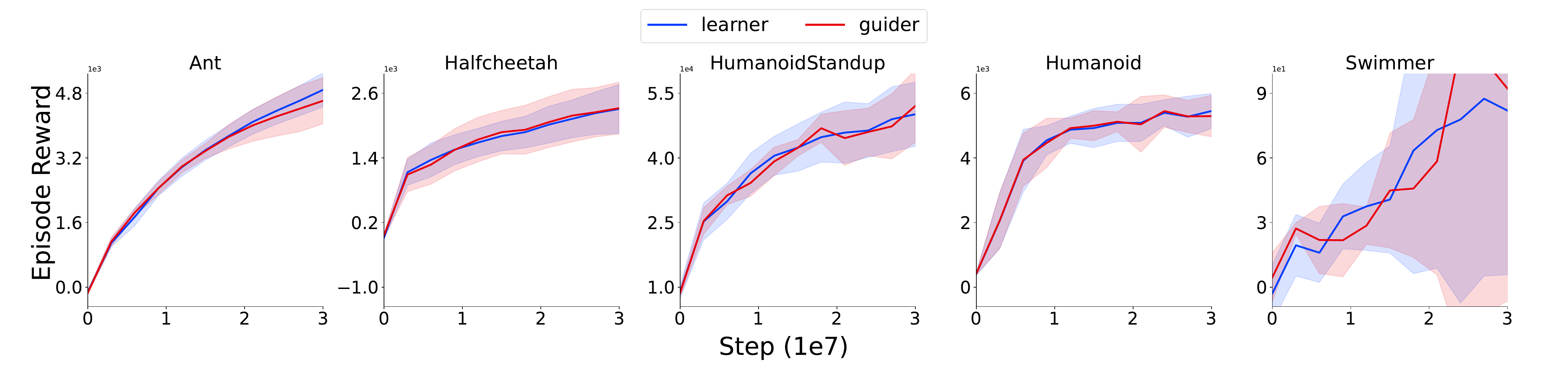}
  \end{subfigure}
  \caption{Comparing the performance of guider and learner of GPO-clip.}
\end{figure}

\subsection{Computational Cost} 
In this section, we compare the computational cost of GPO (both GPO-penalty and GPO-clip share the same cost), PPO-asym, TGRL and the environmental step time across several environments. The results, presented in Table~\ref{app:comp}, show that GPO is approximately 10\% to 20\% slower than PPO-asym.
Importantly, GPO achieves this with no additional networks, underscoring its efficiency despite the modest increase in computational overhead.

\begin{table}[ht] 
\centering 
\caption{Frames Per Second (FPS) of GPO, PPO-asym and TGRL across several environments, computed on the NVIDIA GeForce RTX 4090.} \label{app:comp} 
\begin{tabular}{c|cccc} 
\hline Environment & GPO & PPO-asym & TGRL &Environmental Step\\
\hline
Ant & $1.19 \times 10^5$ & $1.36 \times 10^5$ & $1.13\times 10^2$ & $4.23 \times 10^5$ \\ 
Halfcheetah & $6.27 \times 10^4$ & $7.21 \times 10^4$ & $9.58\times 10^1$ &$2.55 \times 10^5$ \\
Humanoid & $6.29 \times 10^4$ & $7.18 \times 10^4$ & $9.92\times 10^1$& $2.50 \times 10^5$ \\
Swimmer & $3.33 \times 10^4$ & $3.83 \times 10^4$ & $1.04\times 10^2$& $1.50 \times 10^5$ \\
\hline 
\end{tabular} 

\end{table}

\

\end{document}